\documentclass[hidelinks,accepted]{uai2021} % after acceptance, for a revised

\usepackage[american]{babel}
\usepackage{natbib} % has a nice set of citation styles and commands
\usepackage{mathtools} % amsmath with fixes and additions
\usepackage{booktabs} % commands to create good-looking tables
\usepackage{tikz} % nice language for creating drawings and diagrams

\usepackage{amsfonts}       % blackboard math symbols
\usepackage{nicefrac}       % compact symbols for 1/2, etc.
\usepackage{microtype}      % microtypography
\usepackage[linesnumbered,ruled,vlined]{algorithm2e}
\usepackage{wrapfig}
\usepackage{tablefootnote}
\usepackage{graphicx}
\usepackage{subcaption}
\usepackage{appendix}

\usepackage{amsmath}
\usepackage{amsfonts}
\usepackage{amssymb}
\usepackage{amsthm}
\usepackage{bm}
\usepackage{bbm}
\usepackage{mathtools}
\usepackage{enumitem}
\usepackage{thmtools,thm-restate}
\usepackage[capitalise]{cleveref}
\usepackage{graphicx}
\usepackage{comment}
\usepackage{xspace}

\newcommand{\paren}[1]{ \left( #1 \right) }
\newcommand{\bracket}[1]{ \left[ #1 \right] }

\theoremstyle{plain}
\newtheorem{lemma}{Lemma}%[section]
\theoremstyle{definition}
\newtheorem{definition}{Definition}%[section]
\newtheorem{assumption}{Assumption}%[section]

\theoremstyle{remark}
\def\AA{\mathcal{A}}\def\BB{\mathcal{B}}

\def\HH{\mathcal{H}}
\def\KK{\mathcal{K}}

\def\XX{\mathcal{X}}

\def\Ebb{\mathbb{E}}

\def\Rbb{\mathbb{R}}

\def\R{\Rbb}

\newcommand{\norm}[1]{\| #1 \| }
\newcommand{\abs}[1]{| #1 |  }
\newcommand{\normx}[1]{\left \| #1 \right\| }

\newcommand{\<}{\langle}
\renewcommand{\>}{\rangle}

\DeclareMathOperator*{\argmin}{arg\,min}

\def\regret{\textrm{Regret}}

\newcommand{\E}{\Ebb}

\usepackage{etoolbox}  % etex is already included
\usepackage{environ}

\makeatletter
\providecommand{\@fourthoffour}[4]{#4}
\newcommand\fixstatement[2][\proofname\space of]{%
	\ifcsname thmt@original@#2\endcsname
		\AtEndEnvironment{#2}{%
			\xdef\pat@label{\expandafter\expandafter\expandafter
				\@fourthoffour\csname thmt@original@#2\endcsname\space\@currentlabel}%
			\xdef\pat@proofof{\@nameuse{pat@proofof@#2}}%
		}%
	\else
		\AtEndEnvironment{#2}{%
			\xdef\pat@label{\expandafter\expandafter\expandafter
				\@fourthoffour\csname #1\endcsname\space\@currentlabel}%
			\xdef\pat@proofof{\@nameuse{pat@proofof@#2}}%
		}%
	\fi
	\@namedef{pat@proofof@#2}{#1}%
}

\newcounter{proofcount}

\NewEnviron{proofatend}{%
	\edef\next{%
		\noexpand\begin{proof}[\pat@proofof\space\pat@label]%
			\unexpanded\expandafter{\BODY}}%
			\global\toks\numexpr\prooftoks+\value{proofcount}\relax=\expandafter{\next\end{proof}}
	\stepcounter{proofcount}}

\def\printproofs{%
	\count@=\z@
	\loop
	\the\toks\numexpr\prooftoks+\count@\relax
	\ifnum\count@<\value{proofcount}%
	\advance\count@\@ne
	\repeat}
\makeatother

\fixstatement{lemma}
\fixstatement{theorem}
\fixstatement{proposition}
\fixstatement{corollary}

\newcommand{\blue}[1]{{\color{black}{#1}}}

\def\exloss{{l}}  % expected loss
\def\exlossn{{\exloss_n(\theta_n)}}  % expected loss
\def\exlossx{{\exloss_n(\theta)}}
\def\emloss{{\hat l}}  % empirical loss
\def\emlossn{{\emloss_n(\theta_n)}}  % empirical loss
\def\emlossx{{\emloss_n(\theta)}}
\def\exgrad{{\nabla \exlossn}}
\def\emgrad{{\nabla \emlossn}}
\def\gn{{\nabla \emlossn}}
\def\exerr{{\epsilon}}  % expected error
\def\emerr{{\hat \epsilon}}  % empirical loss
\def\embnd{{\hat E}}
\def\exreg{{\regret(\exloss_n)}}
\def\emreg{{\regret(\emloss_n)}}
\def\gradreg{{\regret(\<\emgrad, \cdot\>)}}
\def\expert{{\pi_\mathrm{e}}}
\def\gn{{\nabla f_n(\theta_n)}}
\newcommand{\emrng}[2]{{\emloss_{#1:#2}}}

\crefrangeformat{assumption}{Assumptions #3#1#4,#5#2#6}

\title{Explaining Fast Improvement in Online Imitation Learning}

\author[1]{\href{mailto:Xinyan Yan <voidpointer@gatech.edu>?Subject=Your Paper on Online IL Fast Rate}{Xinyan Yan}{}}
\author[2]{Byron Boots}
\author[3]{Ching-An Cheng}

\affil[1]{%
  Georgia Tech
}
\affil[2]{
  University of Washington
}
\affil[3]{%
Microsoft Research
}

\begin{document}
\maketitle

\begin{abstract}

  Online imitation learning (IL) is an algorithmic framework that leverages interactions with expert policies for efficient policy optimization.
  \blue{
    Here policies are optimized by performing online learning on a sequence of loss functions that encourage the learner to mimic expert actions, and if the online learning has no regret, the agent can provably learn an expert-like policy.
    Online IL has demonstrated empirical successes in many applications and interestingly, its policy improvement speed observed in practice is usually much faster than existing theory suggests.}
  In this work, we provide an explanation of this  phenomenon.
  Let $\xi$ denote the policy class bias and assume the online IL loss functions are convex, smooth, and non-negative. We prove that, after $N$ rounds of online IL with stochastic feedback, the policy improves in $\tilde{O}(1/N + \sqrt{\xi/N})$ in both expectation and high probability.
  In other words, we show that adopting a sufficiently expressive policy class in online IL has two benefits: both the policy improvement speed increases and the performance bias decreases.

\end{abstract}

\section{INTRODUCTION}

Imitation learning (IL) is a framework for improving the sample efficiency of policy optimization in sequential decision making.
Unlike reinforcement learning (RL) algorithms that optimize policies purely by trial-and-error, IL leverages \emph{expert policies} in the training time to provide extra feedback signals to aid the policy search (e.g., in the form of supervised learning losses).
These expert policies can represent human demonstrators or resource-intensive engineered solutions which achieve non-trivial performance in the problem domain.
By following the guidance of an expert policy, the learner can avoid blindly exploring the problem space and focus on promising directions that lead to expert-like behaviors, so the learning becomes sample efficient.

Online IL, pioneered by~\citet{ross2011reduction}, is one of the algorithms that exploit such expert policies.
% to achieve theoretical guarantees and empirical successes.
Given access to interact with an expert policy, online IL reduces policy optimization into no-regret online learning~\citep{hazan2016introduction} for which effective algorithms have been developed.
The main idea of online IL is to design an \emph{online learning problem}\footnote{The online decision in the iterative process of online learning should not be confused with the decisions made at each time step in sequential decision making.} of which
\emph{1)} the decision set is identified with the policy class in the original policy optimization problem; and
\emph{2)}
the online loss functions are set to encourage the learner to take expert-like actions under its own state distribution, which resemble \emph{a sequence of supervised learning problems}.
%and enforce that taking expert-like actions implies expert-like performance on the sequential decision making problem.
When these two conditions are met, the reduction follows: the regret rate and the minimum cumulative loss witnessed in the online learning problem determine respectively the \emph{learning speed} and the \emph{performance bias} in the original policy optimization problem.

Since the seminal work by~\citet{ross2011reduction} was published, significant progress has been made in both theory and practice.
It is shown that, for certain problems, online IL can learn the optimal policy exponentially faster than any RL algorithm when the expert policy is optimal~\citep{sun2017deeply}.
Furthermore, online IL has been validated on physical robot control tasks~\citep{ross2013learning,pan2018agile}.
Beyond typical IL scenarios, online IL has also been applied to design algorithms for system identification~\citep{venkatraman2014data}, model-based RL~\citep{ross2012agnostic},  structured prediction~\citep{ross2014reinforcement,chang2015learning,sun2017deeply}, and combinatorial search~\citep{song2018learning}.
Here we collectively call these algorithms \emph{online IL}, since they adopt the same reduction idea and mainly differ in the way the expert policy is defined.

Despite the success of online IL, there is a mismatch between provable theoretical guarantees and the learning phenomenon observed in practice.
Because of the design constraint imposed on the online losses mentioned above, the online losses used in the online IL reduction are not fully adversarial, but generated by samples of a sequence of probability distributions that vary slowly as the learner updates its policy~\citep{cheng2018convergence}. %\citep{cheng2019online}.
This structure makes the performance guarantee given by the classic adversary-style analysis of the regret rate taken by~\citet{ross2011reduction} overly conservative, and motivates a deeper study on theoretical underpinnings of online IL~\citep{cheng2018convergence,cheng2019accelerating,cheng2019online,lee2019continuous}.

In this work, we are interested in explaining the fast policy improvement of online IL that is observed in practice but not captured by existing theory.
When the online loss functions are convex and Lipschitz, typical analyses of regret and martingales~\citep{ross2011reduction,cesa2004generalization} suggest an on-average convergence rate in $O(1/\sqrt{N})$ after $N$ rounds.
However, empirically, online IL algorithms learn much faster; e.g., the online IL algorithm DAgger~\citep{ross2011reduction} learned to mimic a model predictive control policy for autonomous off-road driving in only three rounds in~\citep{pan2018agile}.
Although the convergence rate improves to $\tilde{O}(1/N)$ when the online losses are strongly convex~\citep{cheng2018convergence},
this condition can be difficult to satisfy especially when the policy class is large, such as a linear function class built on high-dimensional features.
The empirical effectiveness and sample efficiency of online IL demand alternative explanations.

\blue{
  In this work, we bring a new perspective on the efficacy of online IL:
  even when learning from convex (but not strongly convex) sampled online losses, % whose distributions change non-stationarirly from round to round,
  the learner in online IL can actually achieve a $\tilde{O}(1/N)$-like rate, because the consistency that the expert to imitate is fixed across different rounds provides a stability effect to learning.
}
Formally, we prove a new bias-dependent convergence rate for online IL that is adaptive to the performance of the best policy in the policy class on the sequence of sampled losses.
Interestingly, this new rate shows that an online IL algorithm can learn faster as this performance bias becomes smaller.
In other words, adopting a sufficiently expressive policy class in online IL has two benefits:
as the policy class becomes \emph{reasonably} \blue{but not overly} rich, both the \emph{learning speed} increases and the \emph{performance bias} decreases.

\blue{Concretely, suppose that the losses in online IL are convex, smooth, and non-negative, which, e.g., includes learning linear policies with quadratic losses as commonly used in continuous control problems.}
Let $\xi$ denote the policy class bias, \blue{which measures the performance of the best policy in the policy class on the sequence of imitation losses}. We give a convergence rate in $\tilde{O}(1/N + \sqrt{\xi/N})$ both in expectation and in high probability for online IL algorithms using stochastic feedback.
This new result shows a transition from the faster rate of $\tilde{O}(1/N)$ to the usual rate of $\tilde{O}(1/\sqrt{N})$ as the policy class bias $\xi$ increases.
%

%Due to the dependency on the best policy in the hindsight, this type of problem-dependent rates is called optimistic rates, and
This type of bias-dependent or optimistic convergence rate has been studied in
%other machine learning literature.
typical machine learning settings, e.g.,
%The policy class bias resembles the Bayes risk
% When the loss functions are convex, smooth, and non-negative,
%Similar rates have been shown in
statistical learning~\citep[Theorem 1]{srebro2010smoothness}, stochastic convex optimization~\citep{zhang2017empirical,liu2018fast}, and online learning ~\citep[Theorem 2]{srebro2010smoothness}, \citep[Theorem 4.21]{orabona2019modern}.
% Inspired by these results, we derive the bias-dependent improvement rate for online IL.
\blue{
  In fact, our new rate in expectation for online IL can be treated, from a technical viewpoint, as a direct consequence of the bias-dependent bound in the online learning literature.
  However, deriving such a new rate also in \emph{high probability} requires extra technicalities, because the losses in online IL mix non-stationarity and stochasticity together;}
% The mix of non-stationarity and stochasticity in the online loss functions of online IL distinguishes the setup here and makes the analysis challenging;
indeed, previous analyses tackle only one of these two properties and a straightforward combination does not lead to the fast rate desired here (cf. \cref{sec:high probability bound}).
\blue{To prove the desired fast high-probability bound, we propose a new regret decomposition technique for analyzing online IL} and leverage a recent martingale concentration result based on path-wise statistics~\citep[Theorem 3]{rakhlin2015equivalence}.
%to prove our theorems, we resort to a recent result on vector-valued martingale concentration that only depends on pathwise statistics.

We conclude by corroborating the new theoretical findings with experimental results of online IL. % in simulated robot control tasks.
The detailed proofs for this paper can be found in the Appendix.

\section{BACKGROUND: ONLINE IL}

\subsection{Policy Optimization} \label{sec:policy optimization}

The objective of policy optimization is to find a high-performance policy in a policy class $\Pi$ for sequential decision making problems. Typically, it models the world as a Markov decision process (MDP), defined by an initial state distribution, transition dynamics, and an instantaneous state-action cost function~\citep{puterman2014markov}.
This MDP is often assumed to be unknown to the learning agent; therefore the learning algorithm for policy optimization needs to perform systematic exploration in order to discover good policies in $\Pi$.
Concretely, let us consider a policy class $\Pi$ that has a one-to-one mapping to a parameter space $\Theta$, and let $\pi_\theta$ denote the policy associated with the parameter $\theta \in \Theta$. That is, $\Pi = \{\pi_\theta: \theta\in\Theta\}$.
The goal of policy optimization is to find a policy $\pi_\theta\in\Pi$ that minimizes the expected cost,
\begin{align} \label{eq:rl objective}
  J(\pi) \coloneqq \E_{s\sim d_{\pi_\theta}} \E_{a \sim \pi_\theta}[c(s, a)],
\end{align}
where $s$ and $a$ are the state and the action, respectively, $c$ is the instantaneous cost function and $d_{\pi_\theta}$ denotes the \emph{average state distribution} over the problem horizon induced by executing policy $\pi_\theta$ starting from a state sampled from the initial state distribution.
The problem formulation in \eqref{eq:rl objective} applies to various settings of problem horizon and discount rate,
% both infinite and finite horizon problems,
where the main difference is how the average state distribution is defined; e.g., for a discounted problem, $d_{\pi_\theta}$ is defined by a geometric mean, whereas $d_{\pi_\theta}$ is the stationary state distribution for average infinite-horizon problems.

\subsection{Online IL Algorithms} \label{sec:reducing policy optimization}

Online imitation learning (IL) is a policy optimization technique that leverages interactive experts to efficiently find good policies.
It devises a sequence of online loss functions $\exloss_n$ such that \emph{no regret} and \emph{small policy class bias} imply good policy performance in the original sequential decision problem.

Concretely, let $\expert$ be an interactive expert policy.
Instead of minimizing \eqref{eq:rl objective} directly, online IL minimizes \blue{a surrogate objective that upper bounds}
%an upper bound of
the performance difference between the policy $\pi_\theta$ and the expert $\expert$:
\begin{align} \label{eq:il surrogate}
  J(\pi_\theta) - J(\expert) \leq O \Big( \underbrace{\E_{s \sim d_{\pi_\theta}} \E_{a \sim \pi_\theta}[D_{\expert}(s,a)]}_{\text{\blue{surrogate objective}}} \Big),
\end{align}
where the function $D_{\pi_\mathrm{e}}(s,a)$ represents how similar an action $a$ is to the action taken by expert policy $\expert$ at state $s$, measured by statistical distances (e.g., Wasserstein distance and KL divergence) or their upper bounds~\citep{ross2011reduction,ross2014reinforcement,sun2017deeply}. %,cheng2019accelerating}.

Although the surrogate objective in \eqref{eq:il surrogate} resembles \eqref{eq:rl objective} (i.e., by replacing $D_{\pi_\mathrm{e}}(s,a)$ with $c(s,a)$), the surrogate objective  has an additional \blue{critical property that its range is normalized~\citep{cheng2018convergence}: regardless of the definition of the cost function $c$ of the original sequential decision problem}, if the policy class $\Pi$ has enough capacity to contain the expert policy $\expert$, there is a policy $\pi_\theta \in \Pi$ such that, for \emph{all} states,
\begin{align} \label{eq:realizable assumption}
  \E_{a \sim \pi_\theta}[D_{\expert}(s,a)] =0.
\end{align}
%the cumulative loss witnessed by the policy class is zero.
% Leveraging the realizability property in \eqref{eq:realizable assumption},
\blue{Under the realizability assumption \eqref{eq:realizable assumption},}
online IL can minimize the surrogate function in \eqref{eq:il surrogate} by solving an online learning problem: Let parametric space $\Theta$ be the decision set (i.e., the policy class) in online learning; it defines the online loss function in round $n$ as
\begin{align} \label{eq:IL online loss}
  \exlossx =\E_{s \sim d_{\pi_{\theta_n}}} \E_{a \sim \pi_\theta}[D_{\expert}(s, a)],
\end{align}
where $\theta_n\in\Theta$ is the online decision made by the online algorithm in round $n$.

The main benefit of this indirect \blue{iterative} approach is that, compared with the surrogate function \eqref{eq:il surrogate}, the average state distribution $d_{\pi_{\theta_n}}$ in the online loss function \eqref{eq:IL online loss} is not considered as a function of the policy parameter $\theta$, making the online loss function \eqref{eq:IL online loss} the objective function of a supervised learning problem whose sampled gradient is less noisy than that of the surrogate problem in \eqref{eq:il surrogate}.
Because of the
%realizable property in
\blue{realizability assumption}
\eqref{eq:realizable assumption}, the influence of the policy parameter on the change of the average
state distribution can be ignored here, and the average regret with respect to the online loss functions in \eqref{eq:IL online loss} alone~\citep{ross2011reduction} can upper bound the surrogate function in \eqref{eq:il surrogate}.
% The influence of the policy parameter on the change in the average state distribution can be ignored here, because of the .

When the expert policy $\expert$ is only \emph{nearly realizable} by the policy class $\Pi$ (that is, \eqref{eq:realizable assumption} can only be satisfied up to a certain error), optimizing the policy with this online learning reduction would suffer from an extra performance bias due to using a limited policy class, as we will later discuss in \cref{sec:existing guarantees}.

\begin{algorithm}[t]                %<-- Remove float environment
  % \SetCustomAlgoRuledWidth{0.45\textwidth}  %<-- For aesthetics
  \caption{Online Imitation Learning (IL)}
  \label{alg:online imitation learning}
  \KwIn{Initial policy $\pi_{\theta_1}$ and online algorithm $\AA$}
  \KwOut{The best policy in the sequence of policies $\{\pi_{\theta_n}\}_{n=1}^N$}
  Initialize $\AA$ with the initial policy $\pi_{\theta_1}$ \\
  \For{$n$ \textbf{from} $1$ \textbf{to} $N$} {
    Design the online loss function $\exloss_n$ based on $\pi_{\theta_n}$  \\
    Execute  $\pi_{\theta_n}$ in the MDP to gather samples \\
    % Observe an unbiased sample-based estimate  $\emloss_n$ of $\exloss_n$ such that $\E [\emloss_n(\theta)] = \exloss_n(\theta), \forall \theta$ \\
    Use the samples to build an estimate $\emloss_n$ of $\exloss_n$ such that
    for all $\theta$,  $\E [\emloss_n(\theta)] = \exloss_n(\theta)$ \\
    Pass the functional feedback $\emloss_n $ to $\AA$ and use the return of $\AA$ to update policy to $\pi_{\theta_{n+1}}$
    \label{line:online learning update}
  }
\end{algorithm}

\paragraph{Summary}
Online IL can be viewed as a meta algorithm shown in \cref{alg:online imitation learning}, where we take into account that in practice the MDP is unknown and therefore the online loss function $\exloss_n$ needs to be further approximated by finite samples as $\emloss_n$, such that $\forall\theta\in\Theta$, $\E[\emloss_n(\theta)] = \exloss_n(\theta)$.
Given an expert policy, it selects a surrogate function to satisfy conditions similar to \eqref{eq:il surrogate} and \eqref{eq:realizable assumption} (or their approximations).
Then a no-regret online learning algorithm $\AA$ is used to optimize the policy with respect to the sampled online loss functions $\emloss_n$, generating a sequence of policies $\{\pi_{\theta_n}\}_{n=1}^N$. By this reduction, performance guarantees can be obtained for the best policy in this sequence.

\paragraph{Online IL in General}
Before proceeding we note that by following the online IL design protocol above, \cref{alg:online imitation learning} can be instantiated beyond the typical IL setup.
By properly choosing the \emph{definition} of expert policies, the online IL reduction can be used to efficiently solve model-based RL and system identification where the samples of the MDP transition dynamics are treated as experts demonstrations~\citep{ross2012agnostic,venkatraman2014data}, and structured prediction where expert state-action value functions measure how good an action is in the surrogate function in \eqref{eq:il surrogate}~\citep{ross2014reinforcement,sun2017deeply}.
\blue{Similar reduction ideas are also used in recent RL algorithms~\citep{agarwal2019theory,abbasi2019politex}.}

\subsection{Guarantees of Online IL} \label{sec:existing guarantees}
Now that we have reviewed the algorithmic aspects of online IL, we give a brief tutorial of the theoretical foundation of online IL and the known convergence results, which show exactly how regret and policy class bias are related to the performance in the original policy optimization problem.

To this end, let us formally define \emph{1)} the regret and \emph{2)} the policy class bias.
For a sequence of online loss functions $\{f_n\}_{n=1}^N$ and decisions $\{\theta_n\}_{n=1}^N$ in an online learning problem, we define the regret as
\begin{align} \label{eq:regret}
  \textstyle
  \regret(f_n) = \sum f_n(\theta_n) - \min_{\theta \in \Theta} \sum f_n(\theta).
\end{align}
Note that, for brevity, the range in $\sum_{n=1}^N$ is omitted in \eqref{eq:regret} and we will continue to do so below as long as the range is clear from the context.
In addition to the regret, we define two problem-dependent biases of the decision set $\Theta$ (the equivalence of  the policy class $\Pi$).
\begin{definition}[Problem-dependent biases] \label{def:minimum average loss}
  For the sampled loss functions $\{\emloss_n\}_{n=1}^N$ experienced by running \cref{alg:online imitation learning}, we define
  %Let $\emerr$ and $\exerr$ denote the minimum average loss and the minimum average expected loss, respectively, i.e.,
  $\emerr = \frac{1}{N}\min_{\theta \in \Theta} \sum \emlossx$ and $\exerr = \frac{1}{N}\min_{\theta \in \Theta} \sum \exlossx$, where for all $n$ and $\theta$, $\exlossx = \E[\emlossx]$.

\end{definition}

A typical online IL analysis uses the regret and the policy class biases $\exerr$ and $\emerr$ to decompose the cumulative loss $\sum \exlossn$ to provide policy performance guarantees.
Specifically, define $\theta^\star \in \argmin_{\theta \in \Theta} \sum l_n(\theta)$. By \eqref{eq:regret} and \cref{def:minimum average loss}, we can write
\begin{align}
  \textstyle
  \sum \exlossn
   & \textstyle = \emreg +  \paren{\sum \exlossn - \emlossn} + N\emerr \label{eq:regret bound using emerr} \\
   & \textstyle \le \emreg +  \paren{\sum \exlossn - \emlossn}  + \nonumber                                \\
   & \; \; \;
  \paren{\sum \emloss_n(\theta^\star) -  \exloss_n(\theta^\star)} +
  N\exerr \label{eq:regret bound using exerr}
\end{align}
where, in both \eqref{eq:regret bound using emerr} and \eqref{eq:regret bound using exerr}, the first term is the online learning regret, the middle term(s) are the generalization error(s), %(which can concentrate in different rates \cite{kakade2009generalization})
and the last term is the policy class bias.

\blue{Because the surrogate loss $l_n(\theta_n)$ in online IL provides an upper bound on the policy performance in the original sequential decision problem (see \eqref{eq:il surrogate} and \eqref{eq:IL online loss}),
picking the best policy in a policy sequence $\{\pi_{\theta_n}\}_{n=1}^N$ with a small cumulative loss $\sum \exlossn$ guarantees good performance.
}

In a nutshell, existing convergence results of online IL are applications of \eqref{eq:regret bound using emerr} and \eqref{eq:regret bound using exerr} with different upper bounds on the regret  and the generalization errors~\citep{ross2011reduction,ross2012agnostic,ross2014reinforcement,sun2017deeply}.
For example, when the sampled loss functions $\emloss_n$ are bounded, the generalization error(s) (i.e., the middle term(s) in \eqref{eq:regret bound using emerr} and \eqref{eq:regret bound using exerr}) can be bounded by $\tilde{O}(\sqrt{N})$ with high probability by Azuma's inequality (see~\citep{cesa2004generalization} or~\citep[Chapter 9]{hazan2016introduction}).
Together with an $O(\sqrt{N})$ bound on the regret (which is standard for online convex losses)~\citep{hazan2016introduction,mcmahan2017survey}, it implies that the average performance $\frac{1}{N} \sum \exlossn$ and  the best performance $\min_n \exlossn$ converge to $\emerr$ or $\exerr$ at the speed of $\tilde{O}(1/\sqrt{N})$.

However, the rate above often does not explain the fast improvement of online IL observed in practice~\citep{laskey2016robot,sun2017deeply,pan2018agile,cheng2018fast}, as we will also show experimentally in \cref{sec:exps}. While faster rates in $\tilde{O}(1/N)$ was shown for strongly convex loss functions~\citep{ross2011reduction,cheng2018convergence}, %these results are not very assuring either:
the strong convexity assumption usually does not hold; \blue{for example, the common setting of learning with a policy class $\Pi$ and squared losses can easily break the strong convexity assumption, when the state samples are not diverse enough or when the feature dimension is high}. Thus, alternative explanations are needed.

\section{NEW BIAS-DEPENDENT RATES}
\label{sec:main results}

In this section, we present new policy convergence rates that are adaptive to the performance biases in \cref{def:minimum average loss}. The full proof of these theorems is provided in the Appendix. % defined in \eqref{def:minimum average loss}.

\subsection{Setup and Assumptions}
We suppose the parameter space of the policy class $\Theta$ is a closed convex subset of a Hilbert space $\HH$ that is equipped with norm  $\norm{\cdot}$. Since $\norm{\cdot}$ is not necessarily the norm induced by the inner product, we denote its dual norm by $\norm{\cdot}_*$, which is defined as $\norm{x}_* = \max_{\norm{y}=1} \<x, y\>$.

We define admissible algorithms to broaden the scope of online IL algorithms that our analysis covers. % by capturing only the essential requirement.
\begin{definition}[Admissible online algorithm]
  \label{def:admissible}
  We say an online algorithm $\AA$ is \emph{admissible} for a parameter space $\Theta$ if there exists $R_\AA\in[0,\infty)$ such that given any $\eta > 0$ and any sequence of differentiable convex functions $f_n$, $\AA$ can achieve
  %\begin{align} \label{eq:regret as}
  $ \regret(f_n) \le \regret(\<\gn, \cdot\>) \le \frac{1}{\eta} R_\AA^2 + \frac{\eta}{2}\sum \norm{\gn}_*^2$,
  %\end{align}
  where %$g_n = \nabla f_n(\theta_n)$ and
  $\theta_n$ is the decision made by $\AA$ in round $n$.
\end{definition}
% Common online learning algorithms, such as mirror descent or follow-the-regularized leader, are admissible, where $R$ stands for the diameter of the decision set $\Theta$~\citep{hazan2016introduction}.
We assume that \cref{alg:online imitation learning} is realized by an admissible online learning algorithm $\AA$.
This assumption is satisfied by common online algorithms, such as mirror descent~\citep{nemirovski2009robust} and Follow-The-Regularized-Leader~\citep{mcmahan2017survey}, %under first-order or full-information feedback
where $\eta$ in \cref{def:admissible} corresponds to a constant stepsize that is chosen before seeing the online losses, and $R_\AA$ measures that size of the decision set $\Theta$. %A 

Finally, we formally define convex, smooth, and non-negative (CSN) functions; we will assume the online loss $\exloss_n$ in online IL and its sampled version $\emloss_n$ belong to this class.
\begin{definition}[Convex, smooth, and non-negative (CSN) function]  \label{def:CSN}
  A function  $f: \HH \to \R$ is \emph{CSN} if $f$ on $\XX$ is convex, $\beta$-smooth\footnote{A function $f$ is $\beta$-smooth if its gradient $\norm{\nabla f(x) - \nabla f(y)}_* \leq \beta \norm{x-y}$ for $x,y\in\XX$.}, and non-negative.
\end{definition}
Several popular loss functions used in online IL (e.g., squared $\ell_2$-loss and KL-divergence) are indeed CSN (\cref{def:CSN}) (see \cref{sec:case study} for examples). If the losses are not smooth, several smoothing techniques in the optimization literature are available to smooth the losses locally, e.g., Nesterov's smoothing~\citep{nesterov2005smooth}, Moreau-Yosida regularization~\citep{lemarechal1997practical}, and randomized smoothing~\citep{duchi2012randomized}. %(which are used in optimization)

\subsection{Rate in Expectation}
\blue{
  Our first contribution is a \blue{non-symptotic} bias-dependent convergence rate in expectation by analyzing the online regret and the generalization error in the decomposition in \eqref{eq:regret bound using emerr} individually.
  Firstly, under the assumption that sampled losses are CSN (\cref{def:CSN}) and the online algorithm is admissible (\cref{def:admissible}), the online regret can be bounded by extending the bias-dependent regret bound stated for mirror descent~\citep[Theorem 2]{srebro2010smoothness}.
  Secondly, because the generalization error is a martingale difference sequence, it vanishes in expectation.
}
\begin{restatable}{theorem}{thmexpectation}
  \label{thm:rough bound in expectation}
  In \cref{alg:online imitation learning}, suppose $\emloss_n$ is CSN and $\AA$ is admissible.
  Let $\emerr = \frac{1}{N}\min_{\theta \in \Theta} \sum \emlossx$ be the bias, and let $\embnd$ be an upper bound on $\emerr$.
  \blue{Choose the stepsize $\eta$ in  $\AA$ to be}
  % \yan{Actually, the stepsize can not be smaller based on the current proof}
  $\frac{1}{2\left(\beta + \sqrt{\beta^2 + \frac{1}{2} \beta N \embnd R_\AA^{-2} } \right)}$. Then it holds that %then
  % $\frac{1}{2\paren{\beta + \sqrt{\beta^2 + \frac{\beta N \embnd}{2 R_\AA^2}}}}$.
  \begin{align} \label{eq:rate in expectation}
    \textstyle
    \E \left[\frac{1}{N}\sum \exlossn - \emerr\right] \le \frac{8 \beta R_\AA^2}{N} + \sqrt{\frac{8\beta R_\AA^2 \embnd}{N}}
  \end{align}
\end{restatable}

The rate in \eqref{eq:rate in expectation} suggests that an online IL algorithm can learn faster as the policy class bias becomes smaller; this is reflected in the transition from the usual rate $O(1/\sqrt{N})$ to the faster rate $O(1/N)$ when the bias goes to zero. %Therefore, adopting a sufficiently expressive policy class in online IL can decrease the bias decreases as well as improve the  learning speed.
Notably, the rate in \eqref{eq:rate in expectation} does not depend on the dimensionality of $\HH$ but only on $R_\AA$, which one can roughly think of as the largest norm in $\Theta$. Therefore, we can increase the dimension of the policy class to reduce the bias (e.g., by using reproducing kernels~\citep{hofmann2008kernel}) as long as the diameter of $\Theta$ measured by norm $\norm{\cdot}$ (e.g., $\ell_2$-norm) stays controlled.

\blue{
  Although the proof of~\cref{thm:rough bound in expectation} is a straightforward extension of the existing bias-dependent regret bounds from online learning literature, \cref{thm:rough bound in expectation} brings a new perspective of online IL, which better explains its fast improvement and suggests directions for designing new algorithms that learn faster.
  As shown in \cref{thm:rough bound in expectation}, the policy learning speed in online IL can be closely connected to the policy class biases in~\cref{def:minimum average loss} which have been used in the online IL literature as a measure of expressivity.

  Importantly, unlike in adversarial online learning, the biases $\hat{\epsilon}$ and $\epsilon$ in online IL is not arbitrarily large, but of constant sizes in most applications.
  %
  %
  % and thus they have been used in the online IL literature as a measure of expressivity.
  For example, consider a popular application of online IL--learning-to-search by imitating a \emph{deterministic} algorithmic expert $\expert$~\citep{bhardwaj2017learning}. Here the goal is to learn a computationally efficient policy in place of the expert policy that relies on intensive computation or information unavailable at test time (e.g., the expert can be a brute-force search algorithm).
  In each round of learning, a problem instance is drawn from a distribution of problems, and the learner would query for the expert's advice for the state it visits in the sampled problem.
  If we consider a deterministic learner policy $\pi_{\theta_n}$ parameterized by $\theta_n\in\Theta$, the sampled online loss can be set as $\hat{l}_n(\theta_n) = (\pi_{\theta_n}(s_n) - \expert(s_n))^2$, where $s_n$ is the sampled state visited by the learner in round $n$.
  In these problems, the stochasticity comes from sampling problem instances and the learner's states. But when the expert policy is contained in the class of approximators, there is some $\theta^*\in\Theta$ such that $\hat{l}_n(\theta^*)=0$ \emph{simultaneously} for all $n$ and all samples, i.e., $\epsilon = \hat \epsilon = 0$.
  % when the expert is almost realizable, $\epsilon$ and $\hat \epsilon$ are small.
  Generally, one can show that $\epsilon$ and $\hat \epsilon$ are at most the losses incurred by the expert policy plus some distance between the expert and the approximator class.
  Therefore, our results show that online IL in most useful cases  roughly has a $O(1/N)$ rate. %, with a small constant associated with the $O(1/\sqrt{N})$ term in rate (\eqref{eq:rate in expectation}).
}

\paragraph{Online IL with Adaptive Stepsizes}
\blue{
  In~\cref{thm:rough bound in expectation}, the bias-dependent rate~\eqref{eq:rate in expectation} holds when the stepsize of the admissible online learning algorithm $\AA$ is appropriately tuned. While this seems to be a limitation of~\cref{thm:rough bound in expectation}, one can show that the rate~\eqref{eq:rate in expectation} still holds if the stepsizes are properly adapted online (e.g., using an AdaGrad rule~\citep{duchi2011adaptive} $\eta_n = \frac{R_\AA}{2 \sqrt{\sum_{i=1}^n \|\nabla f_n(\theta_n)\|_*^2}}$) without knowing the constants $\beta, N, \hat E$.
  This is because an online algorithm with adaptive stepsizes can obtain almost the same regret guarantee as an algorithm that would know the optimal constant stepsize in advance.
  Furthermore, the high-probability bias-dependent rate in~\cref{thm:rough bound in high probability} that will be presented in~\cref{sec:high probability bound} can also be extended to adaptive stepsizes.
  Please find details in~\cref{app:adaptive stepsizes}.
}

\subsection{Rate in High Probability} \label{sec:high probability bound}

Next we show that a similar \blue{\emph{non-asymptotic}}\footnote{\blue{For compactness, we use the big-O notation to hide the constants in the rate; the exact constants can be found in~\cref{sec:proof_thm_2}.}} bias-dependent convergence rate to the rate \eqref{eq:rate in expectation} also holds in high probability.
\begin{restatable}{theorem}{thmhighprob}
  \label{thm:rough bound in high probability}
  Under the same assumptions and setup of \cref{thm:rough bound in expectation},
  further assume that there is $G\in[0,\infty)$ such that, for any $\theta \in \Theta$, $\norm{\nabla \emloss_n(\theta)}_* \le G$.
  %And let $\exbnd$ denote a uniform bound of the $\exerr$, i.e., $\exbnd \ge \max \exerr$.
  For any $\delta < 1/e$, with probability at least $1-\delta$, the following holds
  \begin{align} \label{eq:rate with high probability} \textstyle
    \frac{1}{N}\sum \exlossn - \exerr  = O\paren{\frac{C  \beta R^2}{N} +
      \sqrt{\frac{C \beta  R^2  (\embnd+ \exerr)}{N}}}
  \end{align}
  where $R_\Theta = \max_{\theta \in \Theta} \norm{\theta}, R = \max(1, R_\Theta, R_\AA), C = \log(1/\delta) \log(G R N)$.
\end{restatable}

We remark that %despite the %fact %that %the convergence rate \eqref{eq:rate with high probability} needs a
the uniform bound $G$ on the norm of the gradients only appears in logarithmic terms.
Therefore, this rate stays reasonable when the loss functions have %constant second derivatives but
gradients whose norm grows with the size of $\Theta$, such as the popular squared loss.

To prove \cref{thm:rough bound in high probability},
one may attempt to build on top of the proof of \cref{thm:rough bound in expectation} by applying basic martingale concentration properties on the martingale difference sequences (MDSs) in \eqref{eq:regret bound using emerr}, or devise a similar scheme for \eqref{eq:regret bound using exerr}.
But taking this direct approach will bring back the usual rate of $O(1/\sqrt{N})$. % even though the online regret can be $O(1/N)$.
To the best of our knowledge, sharp concentration inequalities for the counterparts of MDS in other learning settings cannot be adapted here in a straightforward way.
\citet[Theorem 1]{srebro2010smoothness} prove a fast rate for empirical risk minimizer (ERM) in statistical learning. However, their proof is based on local Rademacher complexities, which do not have obvious extension to non-stationary online losses.
\citet{zhang2017empirical} extend the results of \citet{srebro2010smoothness} to  stochastic convex optimization, but the extension relies on an i.i.d. concentration lemma.\footnote{The i.i.d. concentration lemma further depends on a martingale concentration bound~\citep[Theorem 3.4]{pinelis1994optimum}, which relies on an almost-surely upper bound of the second-order statistics. In comparison, the martingale concentration we utilize in this work relies on second-order statistics that are defined on the sample path~\citep[Theorem 3]{rakhlin2015equivalence}.}
\citet{kakade2009generalization} show fast converging excess risk of online convex programming algorithms when the loss function is Lipschitz and strongly convex; relaxing the strong convexity assumptions is the goal of this work.

\paragraph{Convexity Assumption}
\blue{
  Both \cref{thm:rough bound in expectation} and \cref{thm:rough bound in high probability} require that the sampled online losses $\hat l_n$ are convex.
  \emph{1)} On one hand, this convexity assumption appears to be restrictive, because our results cannot explain learning with generic neural networks.
  Nonetheless, expressive linear policy classes that meet the convexity assumption still include many useful cases (such as RKHS~\citep{hofmann2008kernel} and rich feature sets).
  % the convexity assumption is on the online loss functions, instead of the original sequential decision making problems.
  Furthermore, convexity has been a central assumption in almost all online learning paradigms; we did not attempt to address this limitation in this work.
  \emph{2)} On the other hand, the convexity assumption relaxes the strong convexity assumption needed in the online IL literature~\citep{ross2011reduction,cheng2018convergence}.
  This relaxation is important when the number of samples is smaller than the number of policy parameters or when the samples are not diverse enough. In those cases, even if the expected losses $l_n$ are strongly convex, the sampled losses $\hat{l}_n$ may not be strongly convex because the Hessian matrix is singular.
}

\begin{table*}[ht]
  % \small
  \caption{Comparison of different learning settings. Info.: Information about the loss function available to the learning algorithm in each round. ERM: Empirical risk minimization. Partial FB: Partial feedback.
  }
  \centering
  \begin{tabular}{llcccll}
    \toprule
    \bfseries Setting     & \bfseries Info.     & \bfseries  Stochastic & \bfseries Non-stationary & \bfseries Partial FB & \bfseries Estimator & \bfseries Excess loss to minimize                        \\
    \midrule
    Online IL (this work) & $\emloss_n$         & Yes                   & Yes                      & No                   & Online              & $\sum \exloss_n
    (\theta_n) - \min \sum \exloss_n(\theta)$                                                                                                                                                              \\
    Stochastic bandits    & $\emloss(\theta_n)$ & Yes                   & No                       & Yes                  & Online              & $\sum \exloss(\theta_n) - N \min \exloss(\theta)$        \\
    Online learning       & $\exloss_n$         & No                    & Yes                      & No                   & Online              & $\sum \exloss_n(\theta_n) - \min \sum \exloss_n(\theta)$ \\
    Statistical learning  & $\emloss$           & Yes                   & No                       & No                   & ERM                 & $\exloss(\theta_{\mathrm{ERM}}) - \min \exloss(\theta)$  \\
    Online-to-batch       & $\emloss$           & Yes                   & No                       & No                   & Online              & $\sum l(\theta_n) - N \min l(\theta)$                    \\
    \bottomrule
  \end{tabular}
  \label{table:setting comparison}
\end{table*}

\paragraph{Related Work in Learning}
Similar bias-dependent or optimistic rates have been studied extensively in several more typical learning settings such as contextual bandits~\citep{allen2018make}, statistical learning~\citep{panchenko2002some,srebro2010smoothness,zhang2017empirical,liu2018fast}, online learning with adversarial loss sequences~\citep{srebro2010smoothness,orabona2012beyond}, and
online-to-batch conversion~\citep{littlestone1990mistake, cesa2004generalization}.
\blue{
  \cref{table:setting comparison} summarizes these different learning setups.
  In contrast to bandit settings that focus on discrete actions or simplex geometry, online IL usually leads to online convex losses, a general compact convex decision set, and stochastic functional or gradient feedback.
  Compared to statistical and online learning, online IL concerns loss functions that are both stochastic and online; we can view statistical and online learning as special cases of online IL. The interactions between noises and non-stationarity make the analysis of online IL especially interesting.
}

\paragraph{Specialization to Stochastic Convex Optimization}
Because of the generality of the online IL, an online IL algorithm (\cref{alg:online imitation learning}) running on a \emph{stationary} loss function can serve as a one-pass learning algorithm for stochastic optimization; that is, we have $\exloss_n = \exloss$ for some $\exloss$ for all $n$; % and therefore the online IL setting strictly subsumes the stochastic optimization setting.
By specializing \cref{thm:rough bound in expectation} and \cref{thm:rough bound in high probability} to the stochastic optimization setting, we can recover the existing bounds in the stochastic optimization literature, i.e., Corollary 3 and Theorem 1 in \citep{srebro2010smoothness}, respectively.
These special cases can be derived in a straightforward manner due to the relationship between  $\emerr$ and $\exerr$ when the loss function is fixed (i.e., $\exloss_n = \exloss$): \emph{1)}  $\E[\epsilon] = \E[\emerr]$, and \emph{2)} in high probability, $\emerr - \exerr \le O(\sqrt{1/N})$.
However, we note that for general online IL problems, the sizes of $\emerr$ and $\exerr$ are not comparable and $\E[\epsilon] \neq \E[\emerr]$.

\subsection{Proof Sketch for Theorem 2}
We take a different decomposition of the cumulative loss to avoid the usual $O(1/\sqrt{N})$ rate originating from applying  martingale analyses on the MDSs in \eqref{eq:regret bound using emerr} and \eqref{eq:regret bound using exerr}.
%
%Instead of analyzing the MDSs in \eqref{eq:regret bound using emerr} and \eqref{eq:regret bound using exerr},
Here we construct two \emph{new} MDSs in terms of the gradients: recall $\exerr = \min_{\theta \in \Theta} \sum \exlossx$ %(see \cref{def:minimum average loss})
and let $\theta^\star = \argmin_{\theta \in \Theta} \sum \exlossx$. Then by convexity of $\exloss_n$, we can derive
\begin{align} \label{eq:proof first}
   & \quad \sum \exlossn - N \exerr \nonumber                                         \\
   & \le  \sum \underbrace{\langle \exgrad - \emgrad,  \theta_n \rangle}_{\text{MDS}}
  - %\nonumber
  \\
   & \;\;\; \;\;
  \sum \langle \underbrace{\exgrad - \emgrad}_{\text{MDS}}, \theta^\star \rangle
  + %\nonumber \\
  % &\;\;\; \; \;
  \regret(\<\nabla \emloss_n(\theta_n), \cdot\>) %\gradreg
  \nonumber
\end{align}

Our proof is based on analyzing these three terms.
For the MDSs in \eqref{eq:proof first}, we notice that, for smooth and non-negative functions, the squared norm of the gradients can be bounded by its function value.
\begin{restatable}[Lemma 3.1~\citep{srebro2010smoothness}]{lemma}{selfbounding}
  \label{lm:self bounding}
  Suppose a function $f: \HH \to \R$ is $\beta$-smooth and non-negative, then for any $x \in \HH$, $\norm{\nabla f(x)}_*^2 \le 4 \beta f(x)$.
\end{restatable}
\cref{lm:self bounding} enables us to properly control the second-order statistics of the MDSs in \eqref{eq:proof first}. %s, e.g., variance, which is often the main ingredient in deriving shaper concentration results.
By a recent vector-valued martingale concentration inequality that depends only on second-order statistics~\citep[Theorem 3]{rakhlin2015equivalence}, we obtain a self-bounding property for \eqref{eq:proof first} to get fast concentration rate.

Besides analyzing the MDSs, we need to bound the regret to the linear functions defined by the gradients (the last term in \eqref{eq:proof first}).
Since this last term is linear, not CSN, the bias-dependent online regret in the proof of \cref{thm:rough bound in expectation} does not apply.
Nonetheless, because these linear functions are based on the gradients of CSN functions,
we discover that their regret rate actually obeys the exact same rate as the regret to the CSN loss functions.
This is notable because the regret to these linear functions upper bounds the regret to the CSN loss functions.

Combining the bounds on the MDSs and the regret, we obtain the rate in \eqref{eq:rate with high probability}.

\section{CASE STUDIES} \label{sec:case study}
We use two concrete instantiations of the online IL algorithm (\cref{alg:online imitation learning}) to show how the new theoretical results in \cref{sec:main results} improve the existing understanding of the policy improvement speed in these algorithms.

\subsection{Imitation Learning} \label{sec:online imitation learning}

The seminal work on online IL~\citep{ross2011reduction}
has demonstrated successes in solving many real-world sequential decision making problems~\citep{laskey2016robot,laskey2017dart,pan2018agile}.
When the action space is discrete, a popular design choice is to set $D_{\expert}(s,a)$ in \eqref{eq:IL online loss} as the hinge loss~\citep{ross2011reduction}. % (i.e., the loss function used in SVM~\citep{hofmann2008kernel}).
For continuous domains, $\ell_1$-loss becomes a natural alternative for defining $D_{\expert}(s,a)$, which, e.g., is adopted by~\citet{pan2018agile} for autonomous driving.
When the policy is linear in the parameters, one can verify that these loss functions are convex and non-negative, though not strongly convex.
Therefore, existing theorems suggest only an $O(1/\sqrt{N})$ rate, which does not reflect the fast experimental rates~\citep{ross2011reduction,pan2018agile}.

Although our new theorems are not directly applicable to these non-smooth loss functions, they can be applied to a smoothed version of these non-negative convex loss functions.
For instance, applying the Huber approximation (an instantiation of Nesterov's smoothing)~\citep{nesterov2005smooth} to ``smooth the tip'' of these $\ell_1$-like losses yields a globally smooth function with respect to the $\ell_2$-norm.
As the smoothing mainly changes where the loss is close to zero, our new theorems suggest that, when the policy class is expressive enough, learning with these $\ell_1$-like losses would converge in a $\tilde O(1/N)$ rate before the policy gets very close to the expert policy during policy optimization.

\subsection{Interactive System ID for Model-based RL}

Interactive system identification (ID) is a technique that interleaves data collection and dynamics model learning for robust model-based RL.
\citet{ross2012agnostic} show that interactive system ID can be analyzed under the online IL framework, where the regret guarantee implies learning a dynamics model that mitigates the train-test distribution shift problem~\citep{abbeel2005exploration,ross2012agnostic}.
%as shown in \cite{abbeel2005exploration} under the realizable setting and later in \cite{ross2012agnostic} under the agnostic setting.
%
Let $T$ and $T_\theta$ denote the true and the learned transition dynamics, respectively. A common online loss for interactive system ID is
%\begin{align} \label{eq:system identification}
$  \exlossx = \E_{(s,a) \sim \frac{1}{2}{d_{T_{\theta_{n}}}} + \frac{1}{2}\nu } \bracket{D_{s,a}(T_\theta|| T)}$,
%\end{align}
where $D_{s,a}(T_\theta|| T)$ is some distance between $T$ and $T_\theta$ under state $s$ and action $a$, $\nu$ is the state-action distribution of an exploration policy, and $d_{T_{\theta_{n}}}$ is the state-action distribution induced by running an optimal policy with respect to the  model $T_{\theta_n}$.
When the model class is expressive enough to contain the $T$, it holds $\exlossx=0$ for some $\theta\in\Theta$ (cf. \eqref{eq:realizable assumption}).

%Recall the main assumptions for \cref{thm:rough bound in expectation} and \cref{thm:rough bound in high probability} are that the loss functions $\{\emloss_n\}$ are convex, smooth, and non-negative.
Suppose that the states and actions are continuous.
A common choice for $D_{s,a}(T_\theta|| T)$ in learning deterministic dynamics is the squared error $D_{s,a}(T_\theta|| T)= \norm{T_\theta(s,a)- s'}_2^2$~\citep{ross2012agnostic}, where the $s'$ is the next state in the true transition of $T$.
% Deterministic model.
If $T_\theta$ is linear in $\theta$
%, e.g., $T_\theta(s,a) = \<\theta, \phi(s,a)\>$, where $\phi$ is the feature map,
or belongs to a reproducing kernel Hilbert space, the sampled loss function $\emloss_n$ is CSN.
% Probabilistic model.
Alternatively, when learning a probabilistic model,
$D_{s,a}$ can be selected as the KL-divergence~\citep{ross2012agnostic}; %: $D_\mathrm{KL}(q|| p) = \E_{s \sim q}[-\log p(s) + \log q(s)]$ in \cite{ross2012agnostic}.
%and $\emloss_n$ resembles the negative log likelihood objective in supervised learning.
it is known that if $T_\theta$ belongs to the exponential family of distributions, %, which includes Gaussian and multinomial distributions,
the KL divergence, and hence $\emloss_n$, are smooth and convex~\citep{wainwright2008graphical}. If the sample size is large enough, $\emloss_n$ becomes non-negative in high probability.

As these online losses are CSN,
%satisfy the assumptions made in \cref{sec:main results},
our theoretical results apply and suggest a convergence rate in $\tilde{O}(1/N)$. On the contrary, the finite sample analysis conducted in~\citep{ross2012agnostic} uses the standard online-to-batch techniques~\citep{cesa2004generalization} and can only give a rate of $O(1/\sqrt{N})$. % for the generalization errors in \eqref{eq:regret bound using emerr} and \eqref{eq:regret bound using exerr}.
Our new results provide a better explanation to justify the fast policy improvement speed observed empirically, e.g.,~\citep[Figure 2]{ross2012agnostic}.

\section{EXPERIMENTAL RESULTS} \label{sec:exps}

Although the main focus of this paper is the new theoretical insights, we conduct experiments to provide evidence that the fast policy improvement phenomenon indeed exists, as our theory predicts. We verify the change of the policy improvement rate due to policy class capacity by running an imitation learning experiment in a simulated CartPole balancing task. Details can be found in~\cref{app:exp}.

\paragraph{MDP setup}
The goal of the  CartPole task is to keep the pole upright by controlling the acceleration of the cart. The start state is a configuration with a small uniformly sampled offset from being static and vertical, and the dynamics is deterministic.
In each time step, if the pole is maintained within a threshold from being upright, the learner receives an instantaneous reward of one; otherwise, the learner receives zero rewards and the episode terminates.
This MDP has a 4-dimensional continuous state space and a 1-dimensional continuous action space.

\paragraph{Expert and learner policies}
%To simulate the online IL task,
We use a neural network expert policy (with one hidden layer of $64$ units and $\mathrm{tanh}$ activation)
%The expert policy is trained with additional Gaussian noise (with zero mean and a learnable variance) on the actions
which is trained using policy gradient with GAE~\citep{schulman2015high} and ADAM~\citep{kingma2014adam}.
We let the learner policy be another neural network that shares the same architecture with the expert policy.\blue{
  When learning only the output layer, we copy the weights of the hidden layer from the expert policy and randomly initialize the weights of the output layer;  we can view the learner as a \emph{linear} policy using the representation of the expert policy.
  When learning the full network, we randomly initialize all the weights and biases.
}

\paragraph{Online IL setup}
We emulate online IL with unbiased and biased policy classes.
\blue{
  To define policy classes with different degrees of bias, we impose $\ell_2$-norm constraints of different sizes on the weights of the learner's output layer.
  To define the unbiased policy class, we lift this $\ell_2$-norm constraint.
}
We select $\exloss_n(\theta) = \E_{s \sim d_{\pi_{\theta_n}}} [H_\mu(\pi_\theta(s) - \expert(s))]$ as the online loss in IL (see \cref{sec:reducing policy optimization}), where $H_\mu$ is the Huber function %(implemented in Keras)
defined as
$H_\mu(x) = \frac{1}{2} x^2$ for $\abs{x} \le \mu$ and $
  \mu \abs{x} - \frac{1}{2}\mu^2$ for $\abs{x} > \mu$.
In the experiments,
$\mu$ is set to $0.05$; as a result, $H_\mu$ is linear when its function value is larger than $0.00125$.
\blue{In the setting of training only the output layer}, because the learner's policy is linear, this online loss is CSN (\cref{def:CSN}) in the unknown weights of the learner.
\blue{We use AdaGrad~\citep{mcmahan2010adaptive,duchi2011adaptive} to optimize the learner policy with constant stepsize $0.01$ and $500$ ($\ln 500 \approx 6.2$) iterations.}

\paragraph{Simulation results}

We compare the results in the unbiased and the biased settings, in terms of how the average loss $\frac{1}{N}\sum_{n=1}^N \exlossn$ changes as the number of rounds $N$ in online learning increases.
\blue{
  We impose $\ell_2$-norm constraints of sizes $\{0.1, 0.12, 0.15\}$ on the weights of the output layer to simulate biased policy classes. For comparison, when training the output layer, the $\ell_2$-norm of the final policy trained without the constraint is about $0.18$; when training the full network, it is about $0.23$.
}
\blue{
  The experimental results are depicted in \cref{fig:exp}.
  To better visualize the rate of improvement, we plot both the $x$- and $y$-axis in log scale, so that the slope of the curves directly represents the rate: if the slope is $-1$, the rate is $O(1/N)$ and if the slope is $-\frac{1}{2}$ the rate is $O(1/\sqrt{N})$.
  In~\cref{fig:last_layer}, only the output layer of the learner policy is trained.
  In this setting, all the assumptions made in our theorems are satisfied.
  It can be seen that when using a larger norm constraint (i.e., smaller bias), the learner policy improvement becomes faster, moving towards $O(1/N)$. %, especially within the first 20 ($\ln 20 \approx 3)$ rounds when the learner policy is clearly in the convex region of the loss function (since $\ln 0.00125 \approx -6.67$).
  The curve with the constraint of $0.10$ in~\cref{fig:last_layer} gets a rate slightly faster $O(1/\sqrt{N})$, likely because the Huber loss is strongly convex near zero.
  Interestingly,
  \cref{fig:full_nn} shows that this phenomenon happens also in training the full network, which does not meet the assumption required in the theory.}

\begin{figure}[t]
  \centering
  \begin{subfigure}[b]{0.47\linewidth}
    \includegraphics[width=.98\linewidth]{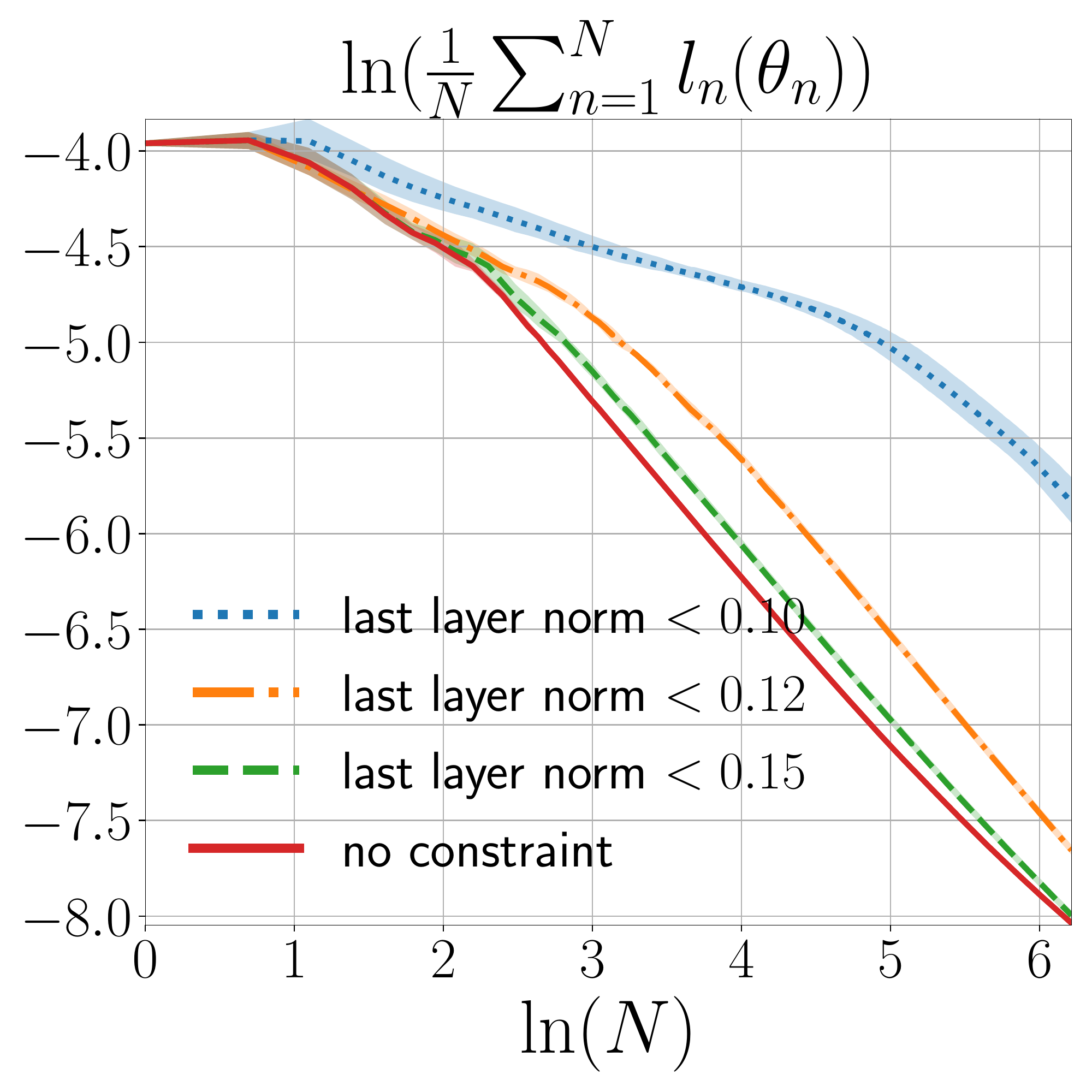}
    \caption{learning output layer}
    \label{fig:last_layer}
  \end{subfigure}
  \begin{subfigure}[b]{0.47\linewidth}
    \includegraphics[width=.98\linewidth]{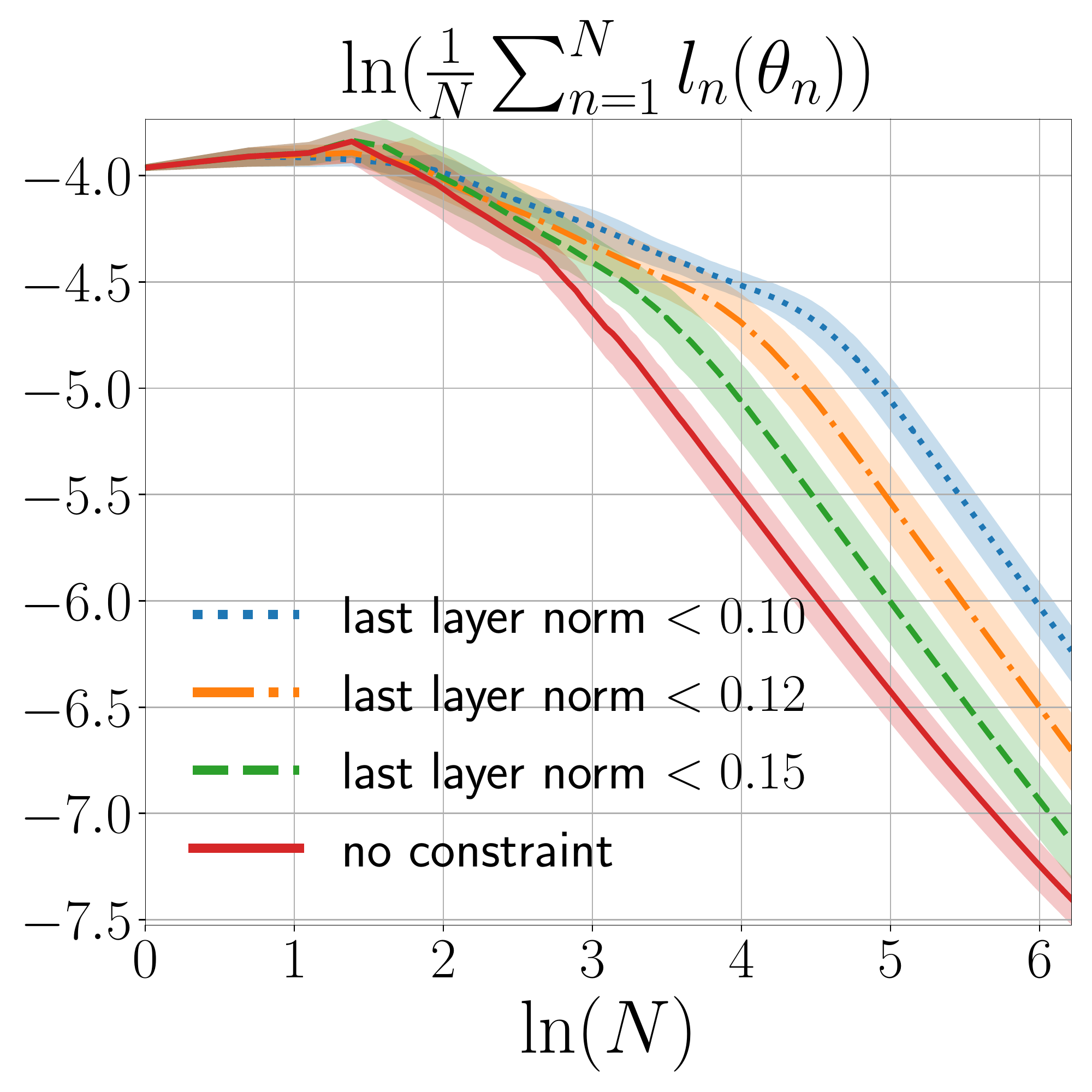}
    \caption{learning full network}
    \label{fig:full_nn}
  \end{subfigure}
  \caption{
    The convergence rate of online IL with different policy class biases, where the bias is defined as the $\ell_2$-norm constraint on the weights of the output layer.
    The curves are plotted using the median over 4 random seeds, and the shaded region represents $10\%$ and $90\%$ percentile.}
  \label{fig:exp}

\end{figure}

\section{CONCLUSION}

In this paper, we provide an explanation of the fast learning speed of online IL by proving new expected and high-probability convergence rates that depend on the policy class capacity.
However, our current results do not explain all the fast improvements of online IL observed in practice.
The analyses here are based on the assumption of using convex and smooth loss functions. This assumption would be violated, for example, with a deep neural network policy based on with ReLU activation; yet~\citet{pan2018agile} show fast empirical convergence rates of these networks in online IL. % even with a deep convolution structure.
Nonetheless, we envision that the insights from this paper can provide a promising direction to better understanding the behaviors of online IL, and to suggest ways for designing new online IL algorithms that proactively leverage these self-bounding regret properties to achieve faster learning.

\bibliography{uai2021-template}

\begin{thebibliography}{49}
\providecommand{\natexlab}[1]{#1}
\providecommand{\url}[1]{\texttt{#1}}
\expandafter\ifx\csname urlstyle\endcsname\relax
  \providecommand{\doi}[1]{doi: #1}\else
  \providecommand{\doi}{doi: \begingroup \urlstyle{rm}\Url}\fi

\bibitem[Abbasi-Yadkori et~al.(2019)Abbasi-Yadkori, Bartlett, Bhatia, Lazic,
  Szepesvari, and Weisz]{abbasi2019politex}
Yasin Abbasi-Yadkori, Peter Bartlett, Kush Bhatia, Nevena Lazic, Csaba
  Szepesvari, and Gell{\'e}rt Weisz.
\newblock Politex: Regret bounds for policy iteration using expert prediction.
\newblock In \emph{International Conference on Machine Learning}, pages
  3692--3702. PMLR, 2019.

\bibitem[Abbeel and Ng(2005)]{abbeel2005exploration}
Pieter Abbeel and Andrew~Y Ng.
\newblock Exploration and apprenticeship learning in reinforcement learning.
\newblock In \emph{Proceedings of the 22nd international conference on Machine
  learning}, pages 1--8, 2005.

\bibitem[Agarwal et~al.(2019)Agarwal, Kakade, Lee, and
  Mahajan]{agarwal2019theory}
Alekh Agarwal, Sham~M Kakade, Jason~D Lee, and Gaurav Mahajan.
\newblock On the theory of policy gradient methods: Optimality, approximation,
  and distribution shift.
\newblock \emph{arXiv preprint arXiv:1908.00261}, 2019.

\bibitem[Allen-Zhu et~al.(2018)Allen-Zhu, Bubeck, and Li]{allen2018make}
Zeyuan Allen-Zhu, S{\'e}bastien Bubeck, and Yuanzhi Li.
\newblock Make the minority great again: First-order regret bound for
  contextual bandits.
\newblock In \emph{International Conference on Machine Learning}, pages
  186--194. PMLR, 2018.

\bibitem[Bhardwaj et~al.(2017)Bhardwaj, Choudhury, and
  Scherer]{bhardwaj2017learning}
Mohak Bhardwaj, Sanjiban Choudhury, and Sebastian Scherer.
\newblock Learning heuristic search via imitation.
\newblock In \emph{Conference on Robot Learning}, pages 271--280. PMLR, 2017.

\bibitem[Brockman et~al.(2016)Brockman, Cheung, Pettersson, Schneider,
  Schulman, Tang, and Zaremba]{brockman2016openai}
Greg Brockman, Vicki Cheung, Ludwig Pettersson, Jonas Schneider, John Schulman,
  Jie Tang, and Wojciech Zaremba.
\newblock Open{AI} {G}ym.
\newblock \emph{arXiv preprint arXiv:1606.01540}, 2016.

\bibitem[Cesa-Bianchi et~al.(2004)Cesa-Bianchi, Conconi, and
  Gentile]{cesa2004generalization}
Nicolo Cesa-Bianchi, Alex Conconi, and Claudio Gentile.
\newblock On the generalization ability of on-line learning algorithms.
\newblock \emph{IEEE Transactions on Information Theory}, 50\penalty0
  (9):\penalty0 2050--2057, 2004.

\bibitem[Chang et~al.(2015)Chang, Krishnamurthy, Agarwal, Langford, and
  Daum{\'e}~III]{chang2015learning}
Kai-Wei Chang, Akshay Krishnamurthy, Alekh Agarwal, John Langford, and Hal
  Daum{\'e}~III.
\newblock Learning to search better than your teacher.
\newblock 2015.

\bibitem[Cheng and Boots.(2018)]{cheng2018convergence}
Ching-An Cheng and Byron Boots.
\newblock Convergence of value aggregagtion for imitation learning.
\newblock In \emph{Proceedings of the 21st International Conference on
  Artificial Intelligence and Statistics}, 2018.

\bibitem[Cheng et~al.(2018)Cheng, Yan, Wagener, and Boots]{cheng2018fast}
Ching-An Cheng, Xinyan Yan, Nolan Wagener, and Byron Boots.
\newblock Fast policy learning through imitation and reinforcement.
\newblock In \emph{Proceedings of the 34th Conference on Uncertanty in
  Artificial Intelligence}, pages 845--855, 2018.

\bibitem[Cheng et~al.(2019{\natexlab{a}})Cheng, Lee, Goldberg, and
  Boots]{cheng2019online}
Ching-An Cheng, Jonathan Lee, Ken Goldberg, and Byron Boots.
\newblock Online learning with continuous variations: Dynamic regret and
  reductions.
\newblock \emph{arXiv preprint arXiv:1902.07286}, 2019{\natexlab{a}}.

\bibitem[Cheng et~al.(2019{\natexlab{b}})Cheng, Yan, Theodorou, and
  Boots.]{cheng2019accelerating}
Ching-An Cheng, Xinyan Yan, Evangelos Theodorou, and Byron Boots.
\newblock Accelerating imitation learning with predictive models.
\newblock In \emph{Proceedings of the 22nd International Conference on
  Artificial Intelligence and Statistics}, 2019{\natexlab{b}}.

\bibitem[Duchi et~al.(2011)Duchi, Hazan, and Singer]{duchi2011adaptive}
John Duchi, Elad Hazan, and Yoram Singer.
\newblock Adaptive subgradient methods for online learning and stochastic
  optimization.
\newblock \emph{Journal of machine learning research}, 12\penalty0 (7), 2011.

\bibitem[Duchi et~al.(2012)Duchi, Bartlett, and
  Wainwright]{duchi2012randomized}
John~C Duchi, Peter~L Bartlett, and Martin~J Wainwright.
\newblock Randomized smoothing for stochastic optimization.
\newblock \emph{SIAM Journal on Optimization}, 22\penalty0 (2):\penalty0
  674--701, 2012.

\bibitem[Hazan et~al.(2016)]{hazan2016introduction}
Elad Hazan et~al.
\newblock Introduction to online convex optimization.
\newblock \emph{Foundations and Trends{\textregistered} in Optimization},
  2\penalty0 (3-4):\penalty0 157--325, 2016.

\bibitem[Hofmann et~al.(2008)Hofmann, Sch{\"o}lkopf, and
  Smola]{hofmann2008kernel}
Thomas Hofmann, Bernhard Sch{\"o}lkopf, and Alexander~J Smola.
\newblock Kernel methods in machine learning.
\newblock \emph{The annals of statistics}, pages 1171--1220, 2008.

\bibitem[Kakade and Tewari(2009)]{kakade2009generalization}
Sham~M Kakade and Ambuj Tewari.
\newblock On the generalization ability of online strongly convex programming
  algorithms.
\newblock In \emph{Advances in Neural Information Processing Systems}, pages
  801--808, 2009.

\bibitem[Kingma and Ba(2014)]{kingma2014adam}
Diederik~P Kingma and Jimmy Ba.
\newblock Adam: A method for stochastic optimization.
\newblock \emph{arXiv preprint arXiv:1412.6980}, 2014.

\bibitem[Laskey et~al.(2016)Laskey, Lee, Chuck, Gealy, Hsieh, Pokorny, Dragan,
  and Goldberg]{laskey2016robot}
Michael Laskey, Jonathan Lee, Caleb Chuck, David Gealy, Wesley Hsieh, Florian~T
  Pokorny, Anca~D Dragan, and Ken Goldberg.
\newblock Robot grasping in clutter: Using a hierarchy of supervisors for
  learning from demonstrations.
\newblock In \emph{2016 IEEE International Conference on Automation Science and
  Engineering (CASE)}, pages 827--834. IEEE, 2016.

\bibitem[Laskey et~al.(2017)Laskey, Lee, Fox, Dragan, and
  Goldberg]{laskey2017dart}
Michael Laskey, Jonathan Lee, Roy Fox, Anca Dragan, and Ken Goldberg.
\newblock Dart: Noise injection for robust imitation learning.
\newblock \emph{arXiv preprint arXiv:1703.09327}, 2017.

\bibitem[Lee et~al.(2018)Lee, Grey, Ha, Kunz, Jain, Ye, Srinivasa, Stilman, and
  Liu]{Lee2018}
Jeongseok Lee, Michael~X. Grey, Sehoon Ha, Tobias Kunz, Sumit Jain, Yuting Ye,
  Siddhartha~S. Srinivasa, Mike Stilman, and C.~Karen Liu.
\newblock {DART}: Dynamic animation and robotics toolkit.
\newblock \emph{The Journal of Open Source Software}, 3\penalty0 (22):\penalty0
  500, feb 2018.

\bibitem[Lee et~al.(2019)Lee, Cheng, Goldberg, and Boots]{lee2019continuous}
Jonathan Lee, Ching-An Cheng, Ken Goldberg, and Byron Boots.
\newblock Continuous online learning and new insights to online imitation
  learning.
\newblock \emph{arXiv preprint arXiv:1912.01261}, 2019.

\bibitem[Lemar{\'e}chal and Sagastiz{\'a}bal(1997)]{lemarechal1997practical}
Claude Lemar{\'e}chal and Claudia Sagastiz{\'a}bal.
\newblock Practical aspects of the moreau--yosida regularization: Theoretical
  preliminaries.
\newblock \emph{SIAM Journal on Optimization}, 7\penalty0 (2):\penalty0
  367--385, 1997.

\bibitem[Littlestone(1990)]{littlestone1990mistake}
Nicholas Littlestone.
\newblock Mistake bounds and logarithmic linear-threshold learning algorithms.
\newblock 1990.

\bibitem[Liu et~al.(2018)Liu, Zhang, Zhang, Jin, and Yang]{liu2018fast}
Mingrui Liu, Xiaoxuan Zhang, Lijun Zhang, Rong Jin, and Tianbao Yang.
\newblock Fast rates of erm and stochastic approximation: Adaptive to error
  bound conditions.
\newblock In \emph{Advances in Neural Information Processing Systems}, pages
  4678--4689, 2018.

\bibitem[McMahan(2017)]{mcmahan2017survey}
H~Brendan McMahan.
\newblock A survey of algorithms and analysis for adaptive online learning.
\newblock \emph{The Journal of Machine Learning Research}, 18\penalty0
  (1):\penalty0 3117--3166, 2017.

\bibitem[McMahan and Streeter(2010)]{mcmahan2010adaptive}
H~Brendan McMahan and Matthew Streeter.
\newblock Adaptive bound optimization for online convex optimization.
\newblock \emph{arXiv preprint arXiv:1002.4908}, 2010.

\bibitem[Nemirovski et~al.(2009)Nemirovski, Juditsky, Lan, and
  Shapiro]{nemirovski2009robust}
Arkadi Nemirovski, Anatoli Juditsky, Guanghui Lan, and Alexander Shapiro.
\newblock Robust stochastic approximation approach to stochastic programming.
\newblock \emph{SIAM Journal on optimization}, 19\penalty0 (4):\penalty0
  1574--1609, 2009.

\bibitem[Nesterov(2005)]{nesterov2005smooth}
Yu~Nesterov.
\newblock Smooth minimization of non-smooth functions.
\newblock \emph{Mathematical programming}, 103\penalty0 (1):\penalty0 127--152,
  2005.

\bibitem[Orabona(2019)]{orabona2019modern}
Francesco Orabona.
\newblock A modern introduction to online learning.
\newblock \emph{arXiv preprint arXiv:1912.13213}, 2019.

\bibitem[Orabona et~al.(2012)Orabona, Cesa-Bianchi, and
  Gentile]{orabona2012beyond}
Francesco Orabona, Nicolo Cesa-Bianchi, and Claudio Gentile.
\newblock Beyond logarithmic bounds in online learning.
\newblock In \emph{Artificial Intelligence and Statistics}, pages 823--831,
  2012.

\bibitem[Pan et~al.(2018)Pan, Cheng, Saigol, Lee, Yan, Theodorou, and
  Boots]{pan2018agile}
Yunpeng Pan, Ching-An Cheng, Kamil Saigol, Keuntak Lee, Xinyan Yan, Evangelos
  Theodorou, and Byron Boots.
\newblock Agile autonomous driving using end-to-end deep imitation learning.
\newblock In \emph{Robotics: science and systems}, 2018.

\bibitem[Panchenko et~al.(2002)]{panchenko2002some}
Dmitriy Panchenko et~al.
\newblock Some extensions of an inequality of vapnik and chervonenkis.
\newblock \emph{Electronic Communications in Probability}, 7:\penalty0 55--65,
  2002.

\bibitem[Pinelis(1994)]{pinelis1994optimum}
Iosif Pinelis.
\newblock Optimum bounds for the distributions of martingales in banach spaces.
\newblock \emph{The Annals of Probability}, pages 1679--1706, 1994.

\bibitem[Puterman(2014)]{puterman2014markov}
Martin~L Puterman.
\newblock \emph{Markov decision processes: discrete stochastic dynamic
  programming}.
\newblock John Wiley \& Sons, 2014.

\bibitem[Rakhlin and Sridharan(2015)]{rakhlin2015equivalence}
Alexander Rakhlin and Karthik Sridharan.
\newblock On equivalence of martingale tail bounds and deterministic regret
  inequalities.
\newblock \emph{arXiv preprint arXiv:1510.03925}, 2015.

\bibitem[Ross and Bagnell(2012)]{ross2012agnostic}
Stephane Ross and J~Andrew Bagnell.
\newblock Agnostic system identification for model-based reinforcement
  learning.
\newblock \emph{arXiv preprint arXiv:1203.1007}, 2012.

\bibitem[Ross and Bagnell(2014)]{ross2014reinforcement}
Stephane Ross and J~Andrew Bagnell.
\newblock Reinforcement and imitation learning via interactive no-regret
  learning.
\newblock \emph{arXiv preprint arXiv:1406.5979}, 2014.

\bibitem[Ross et~al.(2011)Ross, Gordon, and Bagnell]{ross2011reduction}
St{\'e}phane Ross, Geoffrey Gordon, and Drew Bagnell.
\newblock A reduction of imitation learning and structured prediction to
  no-regret online learning.
\newblock In \emph{Proceedings of the fourteenth international conference on
  artificial intelligence and statistics}, pages 627--635, 2011.

\bibitem[Ross et~al.(2013)Ross, Melik-Barkhudarov, Shankar, Wendel, Dey,
  Bagnell, and Hebert]{ross2013learning}
St{\'e}phane Ross, Narek Melik-Barkhudarov, Kumar~Shaurya Shankar, Andreas
  Wendel, Debadeepta Dey, J~Andrew Bagnell, and Martial Hebert.
\newblock Learning monocular reactive uav control in cluttered natural
  environments.
\newblock In \emph{2013 IEEE international conference on robotics and
  automation}, pages 1765--1772. IEEE, 2013.

\bibitem[Schulman et~al.(2015)Schulman, Moritz, Levine, Jordan, and
  Abbeel]{schulman2015high}
John Schulman, Philipp Moritz, Sergey Levine, Michael Jordan, and Pieter
  Abbeel.
\newblock High-dimensional continuous control using generalized advantage
  estimation.
\newblock \emph{arXiv preprint arXiv:1506.02438}, 2015.

\bibitem[Song et~al.(2018)Song, Lanka, Zhao, Bhatnagar, Yue, and
  Ono]{song2018learning}
Jialin Song, Ravi Lanka, Albert Zhao, Aadyot Bhatnagar, Yisong Yue, and
  Masahiro Ono.
\newblock Learning to search via retrospective imitation.
\newblock \emph{arXiv preprint arXiv:1804.00846}, 2018.

\bibitem[Srebro et~al.(2010)Srebro, Sridharan, and
  Tewari]{srebro2010smoothness}
Nathan Srebro, Karthik Sridharan, and Ambuj Tewari.
\newblock Smoothness, low noise and fast rates.
\newblock In \emph{Advances in neural information processing systems}, pages
  2199--2207, 2010.

\bibitem[Sun et~al.(2017)Sun, Venkatraman, Gordon, Boots, and
  Bagnell]{sun2017deeply}
Wen Sun, Arun Venkatraman, Geoffrey~J Gordon, Byron Boots, and J~Andrew
  Bagnell.
\newblock Deeply aggrevated: Differentiable imitation learning for sequential
  prediction.
\newblock In \emph{Proceedings of the 34th International Conference on Machine
  Learning-Volume 70}, pages 3309--3318. JMLR. org, 2017.

\bibitem[Teboulle(2018)]{teboulle2018simplified}
Marc Teboulle.
\newblock A simplified view of first order methods for optimization.
\newblock \emph{Mathematical Programming}, 170\penalty0 (1):\penalty0 67--96,
  2018.

\bibitem[Venkatraman et~al.(2014)Venkatraman, Boots, Hebert, and
  Bagnell]{venkatraman2014data}
Arun Venkatraman, Byron Boots, Martial Hebert, and J~Andrew Bagnell.
\newblock Data as demonstrator with applications to system identification.
\newblock In \emph{ALR Workshop, NIPS}, 2014.

\bibitem[Wainwright and Jordan(2008)]{wainwright2008graphical}
Martin~J Wainwright and Michael~I Jordan.
\newblock Graphical models, exponential families, and variational inference.
\newblock \emph{Foundations and Trends{\textregistered} in Machine Learning},
  1\penalty0 (1-2):\penalty0 1--305, 2008.

\bibitem[Zhang et~al.(2017)Zhang, Yang, and Jin]{zhang2017empirical}
Lijun Zhang, Tianbao Yang, and Rong Jin.
\newblock Empirical risk minimization for stochastic convex optimization: O
  (1/n) -and o (1/n\^2)-type of risk bounds.
\newblock \emph{arXiv preprint arXiv:1702.02030}, 2017.

\bibitem[Zinkevich(2003)]{zinkevich2003online}
Martin Zinkevich.
\newblock Online convex programming and generalized infinitesimal gradient
  ascent.
\newblock In \emph{Proceedings of the 20th international conference on machine
  learning (icml-03)}, pages 928--936, 2003.

\end{thebibliography}
\clearpage
\onecolumn
\appendix
\appendixpage

% \allowdisplaybreaks
% \aistatstitle{Explaining Fast Improvement in Online Imitation Learning}
\section{PROOF OF TOOL LEMMAS}
\subsection{Proof of Lemma 1}
For completeness, we provide the proof for the basic inequality that upper bounds the norm of gradients by the function values, for smooth and nonnegative functions.
This is essential for obtaining the self-bounding properties for proving \cref{lm:ol bias dependent rate} and \cref{thm:rough bound in high probability} later on.
\selfbounding*
\begin{proof}
	Fix any $x \in \HH$. And fix any $y \in \HH$ satisfying $\norm{y-x} \le 1$.
	Let $g(u) = f(x + u(y-x))$ for any $u \in \R$. Fix any $u, v \in \R$,
	\begin{align*}
		|g'(v) - g'(u)| & = |\< \nabla f(x+v(y-x)) - \nabla f(x+u(y-x)), y - x\>|       \\
		                & \le \|\nabla f(x+v(y-x)) - \nabla f(x+u(y-x)) \|_* \| y - x\| \\
		                & \le \beta |v-u|\|y - x\|^2                                    \\
		                & \le \beta |v-u|
	\end{align*}
	Hence, $g$ is $\beta$-smooth.
	By the mean-value theorem, for any $u, v \in \R$, there exists $w \in (u, v)$, such that $g(v) = g(u) + g'(w) (v-u)$. Hence
	\begin{align*}
		0 & \le g(v) = g(u) + g'(u)(v-u) + (g'(w) - g'(u))(v-u)                             \\
		  & \le g(u) + g'(u)(v-u) + \beta |w-u| |v-u| \le g(u) + g'(u)(v-u) + \beta (v-u)^2
	\end{align*}
	Setting $v = u - \frac{g'(u)}{2\beta}$ yields that $|g'(u)| \le \sqrt{4\beta g(u)}$. Therefore, we have
	\begin{align*}
		|g'(0)| = |\< \nabla f(x), y - x\>| \le  \sqrt{4\beta g(0)} = \sqrt{4 \beta f(x)}
	\end{align*}
	Therefore, by the definition of dual-norm,
	\begin{align*}
		\norm{\nabla f(x)}_* = \sup_{y\in \BB, \norm{y-x} \le 1} \< \nabla f(x), y-x \> = \sup_{y\in \BB, \norm{y-x} \le 1} |\< \nabla f(x), y-x \> |\le \sqrt{4 \beta f(x)}
	\end{align*}
	where the second equality is due to the domain of $y-x$.
\end{proof}
It's worthy to note that $f$ needs to be smooth and non-negative on the entire Hilbert space $\HH$.

% \section{Proof of \cref{thm:rough bound in expectation}}
\section{PROOF OF THEOREM 1} \label{sec:proof_thm_1}
\thmexpectation*

The rate \eqref{eq:rate in expectation} follows from analyzing the regret and the generalization error in the decomposition  in \eqref{eq:regret bound using emerr}.
First, under the assumption of CSN loss functions and admissible online algorithms, the online regret can be bounded by an extension of the bias-dependent regret that is stated for mirror descent in~\citep[Theorem 2]{srebro2010smoothness}, whose average gives the rate in \eqref{eq:rate in expectation} (see \cref{app:online regret}).
Second, the generalization error in \eqref{eq:regret bound using emerr} vanishes in expectation because it is a martingale difference sequence (see \cref{app:vanish}).

\subsection{Upper Bound of Online Regret} \label{app:online regret}
We show a bias-dependent regret of admissible online algorithms (\cref{def:admissible}) with CSN functions (\cref{def:CSN}) by extending Theorem 2 of \citep{srebro2010smoothness} as follows.
\begin{lemma} \label{lm:ol bias dependent rate}
	Consider running an admissible online algorithm $\AA$ on a sequence of CSN loss functions $\{f_n\}$.
	Let $\{\theta_n\}$ denote the online decisions made in each round, and let $\emerr = \frac{1}{N}\min_{\theta \in \Theta} \sum f_n(\theta)$ be the bias, and let $\embnd$ be such that $\embnd\geq \emerr$ almost surely.
	Choose $\eta$ for $\AA$ to be
	$\frac{1}{2\paren{\beta + \sqrt{\beta^2 +
					\frac{\beta N \embnd}{2 R_\AA^2}}}}$.
	Then the following holds
	%\cheng{Missing $\frac{1}{N}$ is the definition of $\hat{E}$?}
	\begin{align*}
		\regret(f_n) \le 8 \beta R_\AA^2 +  \sqrt{8\beta R_\AA^2 N \embnd}.
	\end{align*}
\end{lemma}
\begin{proof}
	%\cheng{define what $\gn$ is again or maybe it's better to just write it out here}
	Because the online algorithm $\AA$ is admissible, we have
	\begin{align}
		\regret(f_n)  \le \frac{1}{\eta} R_\AA^2 + \frac{\eta}{2}\sum \norm{\gn}_*^2
	\end{align}
	Let $\lambda = \frac{1}{2\eta}$ and	$r^2 = 2R_\AA^2$, then
	\begin{align}
		\frac{1}{\eta} R_\AA^2 + \frac{\eta}{2}\sum \norm{\gn}_*^2
		= \lambda r^2 + \sum \frac{1}{4\lambda}\norm{\gn}_*^2
	\end{align}
	Using \cref{lm:self bounding} yields a \emph{self-bounding property} for $\regret(f_n)$:
	%\cheng{Missing $\frac{\beta}{\lambda}$ on $\hat{E}$?}
	\begin{align} \label{eq:concrete self-bounding}
		\regret(f_n) \le \lambda r^2 +  \frac{\beta}{\lambda} \sum f_n(\theta_n) \le \lambda r^2 + \frac{\beta}{\lambda} \regret(f_n) + \frac{\beta}{\lambda} N \embnd
	\end{align}
	By rearranging the terms, we have a bias-dependent upper bound
	\begin{align} \label{eq:proof thm1 bound}
		\regret(f_n) \le \frac{\beta}{\lambda - \beta} N\embnd + \frac{\lambda^2}{\lambda - \beta} r^2
	\end{align}
	The upper bound can be minimized by choosing an optimal $\lambda$. Setting the derivative of the right-hand side to zero, and computing the optimal $\lambda$ ($\lambda > 0$) gives us
	\begin{align} \label{eq:lambda}
		r^2 \lambda^2 - 2 \beta r^2 \lambda - \beta N \embnd = 0, \quad \lambda > 0 \quad \text{ and } \quad \lambda = \beta + \sqrt{\beta^2 + \frac{\beta N \embnd}{r^2}}
	\end{align}
	which implies that the optimal $\eta$ is $\frac{1}{2\paren{\beta + \sqrt{\beta^2 + \frac{\beta N \embnd}{2 R_\AA^2}}}}$.
	Since the optimal $\lambda$ satisfies $\beta N \embnd = r^2 \lambda^2 - 2 \beta r^2 \lambda$ implied from \eqref{eq:lambda}, \eqref{eq:proof thm1 bound} can be simplified into:
	\begin{align}
		\regret(f_n) & \le \frac{1}{\lambda-\beta} \beta N \embnd + \frac{\lambda^2}{\lambda - \beta} r^2 =
		\frac{1}{\lambda - \beta}(r^2 \lambda^2 - 2 \beta r^2 \lambda) +\frac{\lambda^2}{\lambda - \beta} r^2 \nonumber \\
		             & =\frac{2\lambda^2 r^2 - 2\beta \lambda r^2}{\lambda - \beta}
		= 2 \lambda r^2
	\end{align}
	Plugging in the optimal $\lambda$ yields
	%\cheng{I don't get the first inequality; why $\hat{\epsilon}$ suddenly appears, what is it? write out what the last inequality is. $\sqrt{a+b}\leq\sqrt{a}+\sqrt{b}$? }
	\begin{align} \label{eq:ol result}
		\regret(f_n) & \le 2 \lambda r^2
		= 2 \paren{\beta + \sqrt{\beta^2 + \frac{\beta N\embnd }{r^2}}} r^2\nonumber             \\
		             & = 2 \beta r^2 + \sqrt{2\beta r^2}  \sqrt{2\beta r^2 + 2N\embnd} \nonumber \\
		             & \le 4 \beta r^2 + 2 \sqrt{\beta r^2 N \embnd} \nonumber                   \\
		             & = 8 \beta R_\AA^2 +  \sqrt{8\beta R_\AA^2 N \embnd}
	\end{align}
	where the last inequality uses the basic inequality: $\sqrt{a+b}\leq\sqrt{a}+\sqrt{b}$.
\end{proof}
Notably, the admissibility defined in \cref{def:admissible} is satisfied by common online algorithms, such as mirror descent \citep{nemirovski2009robust} and Follow-The-Regularized-Leader \citep{mcmahan2017survey} under first-order or full-information feedback, where $\eta$ in \cref{def:admissible} corresponds to a constant stepsize, and $R_\AA$ measures the size of the decision set $\Theta$.
%is the diameter of the decision set $\Theta$.
More concretely, assume that the loss functions $\{f_n\}$ are convex.
Then for mirror descent,  with constant stepsize $\eta$, i.e., $\theta_{n+1} = \argmin_{\theta \in \Theta} f_n(\theta) +  \frac{1}{\eta}D_h(\theta||\theta_n)$, where $h$ is 1-strongly convex and $D_h$ is the Bregman distance generated by $h$ defined by $D_h(x||y) = h(x) - h(y) - \<\nabla h(y), x-y\>$ \citep{teboulle2018simplified},
$R_\AA^2$ can be set to $\max_{x, y\in \Theta} D_h(x||y)$.
And for FTRL with constant stepsize $\eta$, i.e., $\theta_{n+1} = \argmin_{\theta \in \Theta} \sum f_n(\theta) + \frac{1}{\eta} h(\theta)$, where $h$ is 1-strongly convex and non-negative, $R_\AA^2$ can be set to $\max_{\theta \in \Theta} h(\theta)$ \citep[Theorem 1]{mcmahan2017survey}.

\subsection{The Generalization Error Vanishes in Expectation} \label{app:vanish}
The generalization error in \eqref{eq:regret bound using emerr} vanishes in expectation because it is a martingale difference sequence.
\begin{lemma} \label{lm:vanish}
	For \cref{alg:online imitation learning}, the following holds:
	$\E [\sum \exlossn -\sum \emlossn]=0$.
\end{lemma}
\begin{proof}
	We show this by working from the end of the sequence. For brevity, we use the symbol colon in the subscript to represent a set that includes the start and the end indices, e.g. $\emloss_{1:N-2}$ stands for $\{\emloss_1, \dots, \emloss_{N-2}\}$.
	\begin{align*}
		\E_\emrng{1}{N} \left[\sum_{t=1}^N \exlossn \right]
		 & =  \E_\emrng{1}{N-1} \bracket{\sum_{t=1}^{N-1} \exlossn + \exloss_N(\theta_N)}                                                                                                                      \\
		 & = \E_\emrng{1}{N-1} \bracket{\sum_{t=1}^{N-1} \exlossn + \E_{\emloss_N|\emrng{1}{N-1}} \bracket{\emloss_N(\theta_N)}}                                                                               \\
		 & =\E_\emrng{1}{N-2} \bracket{\sum_{t=1}^{N-2} \exlossn + \exloss_{N-1}(\theta_{N-1}) + \E_{\emrng{N-1}{N}|\emrng{1}{N-2}}\bracket{\emloss_N(\theta_N)}}                                              \\
		 & = \E_\emrng{1}{N-2} \bracket{\sum_{t=1}^{N-2} \exlossn + \E_{\emloss_{N-1}|\emrng{1}{N-2}} \bracket{\emloss_{N-1}(\theta_{N-1})} + \E_{\emrng{N-1}{N}|\emrng{1}{N-2}}\bracket{\emloss_N(\theta_N)}} \\
		 & = \E_\emrng{1}{N-2} \bracket{\sum_{t=1}^{N-2} \exlossn +\E_{\emrng{N-1}{N}|\emrng{1}{N-2}} \bracket{\sum_{t=N-1}^N \emlossn}}
	\end{align*}
	By applying the steps above repeatedly, the desired equality can be obtained.
\end{proof}
\subsection{Putting Together}
Finally, plugging \cref{lm:ol bias dependent rate} and \cref{lm:vanish} into \eqref{eq:regret bound using emerr} yields \eqref{eq:rate in expectation}.

% XX Hack
% \section{PROOF OF \cref{thm:rough bound in high probability}}
\section{PROOF OF THEOREM 2} \label{sec:proof_thm_2}
\thmhighprob*

\subsection{Decomposition}

The key to avoid the slow rate due to the direct application of martingale concentration analyses on the MDSs in \eqref{eq:regret bound using emerr} and \eqref{eq:regret bound using exerr} is to take a different decomposition of the cumulative loss.
%
%Instead of analyzing the MDSs in \eqref{eq:regret bound using emerr} and \eqref{eq:regret bound using exerr},
Here we construct two \emph{new} MDSs in terms of the gradients: recall $\exerr = \min_{\theta \in \Theta} \sum \exlossx$ %(see \cref{def:minimum average loss})
and let $\theta^\star = \argmin_{\theta \in \Theta} \sum \exlossx$. Then by convexity of $\exloss_n$, we can derive
\begin{align} \label{eq:proof first app}
	 & \quad \sum \exlossn - N \exerr \nonumber                                                                        \\
	 & \le \sum \<\exgrad, \theta_n - \theta^\star \> \nonumber                                                        \\
	 & = \sum  \<\exgrad - \emgrad, \theta_n - \theta^\star \> + \sum \<\emgrad, \theta_n - \theta^\star \>  \nonumber \\
	 & \le  \sum \underbrace{\langle \exgrad - \emgrad,  \theta_n \rangle}_{\text{MDS}}
	- \sum \langle \underbrace{\exgrad - \emgrad}_{\text{MDS}}, \theta^\star \rangle
	+ \regret(\<\nabla \emloss_n(\theta_n), \cdot\>)  %\gradreg
\end{align}

Our proof is based on analyzing these three terms.
The two MDSs are analyzed in \cref{app:mds} and the regret is analyzed in \cref{app:regret}.

\subsection{Upper Bound of the Martingale Concentration} \label{app:mds}

For the MDSs in \eqref{eq:proof first app}, we notice that, for smooth and non-negative functions, the squared norm of the gradient can be bounded by the corresponding function value through \cref{lm:self bounding}. This enables us to properly control the second-order statistics of the MDSs in \eqref{eq:proof first app}. %s, e.g. variance, which is often the main ingredient in deriving shaper concentration results.
By a recent vector-valued martingale concentration inequality that depends only on the second-order statistics \citep{rakhlin2015equivalence}, we obtain a self-bounding property for \eqref{eq:proof first app} to get a fast concentration rate.
The martingale concentration inequality is stated in the following lemma.
%\footnote{The Bregman distance is defined by $D_h(x||y) = r(\theta) - r(\theta') - \<\nabla r(\theta'), x-y\>$.}
\begin{lemma} [Theorem 3 \citep{rakhlin2015equivalence}] \label{lm:pathwise martingale bound}
	Let $\KK$ be a Hilbert space with norm $\norm{\cdot}$ whose dual is $\norm{\cdot}_*$.
	Let $\{z_t\}$ be a $\KK$-valued martingale difference sequence with respect to $\{y_t\}$, i.e., $\E_{z_t|y_1, \dots, y_{t-1}} [z_t] = 0$, and
	let $h$ be a 1-strongly convex function with respect to norm $\norm{\cdot}$ and let $B^2 = \sup_{x, y \in \KK, \norm{x}=1, \norm{y}=1}D_{h}(x||y)$.
	%$B^2 = \sup_{\norm{x}_*\le 1, \norm{y}_* \le 1} \norm{x}^2_* - \norm{y}^2_* - \langle\nabla \norm{y}^2_*, {x-y}\rangle$ .
	Then for  $\delta \le 1/e$, with probability at least $1-\delta$, the following holds
	\begin{align*}
		\normx{\sum z_t}_* \le 2 B\sqrt{V} + \sqrt{2 \log(1/\delta)} \sqrt{1+1/2\log(2V+2W+1)}
		\sqrt{2V + 2W + 1}
	\end{align*}
	where $V = \sum \norm{z_t}_*^2$ and $W = \sum \E_{z_t|y_1, \dots, y_{t-1}} \norm{z_t}_*^2$.
\end{lemma}
In order to apply \cref{lm:pathwise martingale bound} to the MDSs in \eqref{eq:proof first app}, the key is to properly upper bound the statistics $V$ and $W$ in \cref{lm:pathwise martingale bound} for these MDSs.

\subsubsection{Upper Bound of the Concetration for MDS $\langle \exgrad - \emgrad, \theta_n  \rangle$}
Suppose that the decision set $\Theta$ is inside a ball centered at the origin in $\HH$ with radius $R_\Theta$.
\begin{assumption} \label{as:R}
	There exists $R_\Theta \in [0, \infty)$, such that
	$\max_{\theta \in \Theta} \norm{\theta} \le R_\Theta$.
\end{assumption}
Then by the definition of $V$ and $W$ in \cref{lm:pathwise martingale bound}, and the definitions of the two problem-dependent policy class biases $\exerr$ and $\emerr$ (see \cref{def:minimum average loss}), one can obtain
\begin{align}
	V & =\sum |\langle \exgrad - \emgrad, \theta_n  \rangle|^2                  & \nonumber                                                            \\
	  & \le \sum R_\Theta^2 \norm{ \exgrad - \emgrad}_*^2                       & \text{\cref{as:R}}\nonumber                                          \\
	  & \le \sum R_\Theta^2 \paren{2 \norm{\exgrad}_*^2 + 2 \norm{\emgrad}_*^2} & \text{triangle inequality} \label{eq:V for x_t before self bounding} \\
	  & \le \sum R_\Theta^2 \paren{8 \beta \exlossn + 8 \beta \emlossn}         & \text{\cref{lm:self bounding}}\nonumber                              \\
	  & = 8\beta R_\Theta^2 (\exreg+ \emreg+ N\exerr + N\emerr)                 & \text{\cref{def:minimum average loss}} \label{eq:V}
\end{align}
Similarly  for $W$, we have
\begin{align}
	W & = \sum \E_{\emloss_n|\theta_n}[ |\langle \exgrad - \emgrad, \theta_n \rangle|^2] \nonumber                                                            \\
	  & \le\sum R_\Theta^2 \E_{\emloss_n|\theta_n} [\norm{\exgrad - \emgrad}_*^2]                        & \text{\cref{as:R}}
	\nonumber                                                                                                                                                 \\
	  & \le \sum R_\Theta^2 \paren{2\norm{\exgrad}_*^2 + 2 \E_{\emloss_n|\theta_n}[\norm{\emgrad}_*^2] } & \text{triangle inequality}
	\label{eq:W for x_t before self bounding}                                                                                                                 \\
	  & \le \sum R_\Theta^2 \paren{8 \beta \exlossn + 8  \E_{\emloss_n|\theta_n} [\beta \emlossn] }
	  & \text{\cref{lm:self bounding}} \nonumber                                                                                                              \\
	  & = \sum R_\Theta^2 \paren{8 \beta \exlossn + 8 \beta \exlossn } \nonumber                                                                              \\
	  & = 16 \beta R_\Theta^2 (\exreg + N\exerr)                                                         & \text{\cref{def:minimum average loss}}\label{eq:W}
\end{align}
Therefore,
\begin{align} \label{eq:v_w_1}
	V+W \le 24 \beta R_\Theta^2(\exreg + \emreg + N \exerr + N \emerr)
\end{align}
Further suppose that the gradient of the sampled loss can be uniformly bounded:
\begin{assumption} \label{as:G}
	For any loss sequence $\{\emloss_n\}$ that can be experienced by \cref{alg:online imitation learning}, suppose that there is $G\in[0,\infty)$ such that, for any $\theta \in \Theta$, $\norm{\nabla \emloss_n(\theta)}_* \le G$.
\end{assumption}
Then due to \eqref{eq:v_w_1}, $V \le 4G^2R_\Theta^2 N$ and $W \le 4G^2 R_\Theta^2 N$.
Now we are ready to invoke \cref{lm:pathwise martingale bound} by letting the Hilbert space $\KK$ in \cref{lm:pathwise martingale bound} be $\R$, and denoting the corresponding $B$ in \cref{lm:pathwise martingale bound} by $B_\R$.
Then, for $\delta > 1/e$, with probability at least $1-\delta$, the following holds
\begin{align} \label{eq:expected regret for x_t}
	 & \quad \abs{\sum \langle \exgrad - \emgrad, \theta_n  \rangle } \nonumber                                                                         \\
	 & \le 2B_\R \sqrt{8\beta R_\Theta^2}\sqrt{ \exreg + \emreg + N\exerr + N\emerr}  \; +
	\nonumber                                                                                                                                           \\
	%   & \sqrt{96\beta R_\Theta^2 \log(1/\delta)} \sqrt{1+1/2\log(16G^2R_\Theta^2 N+1)} \sqrt{\exreg+ N\exerr +\emreg/3 + N\emerr/3+ 1/48}
	 & \sqrt{96\beta R_\Theta^2 \log(1/\delta)} \sqrt{1+1/2\log(16G^2R_\Theta^2 N+1)} \sqrt{\exreg+ N\exerr +\emreg + N\emerr+ 1/(48 \beta R_\Theta^2)}
\end{align}

\subsubsection{Upper Bound of the Concetration for MDS $\exgrad - \emgrad$}
To bound $\|\sum \exgrad - \emgrad\|_*$ that appears in
\begin{align} \label{eq:x_star cauchy}
	\sum \langle \exgrad - \emgrad, \theta^\star \rangle
	\le R_\Theta \|\sum \exgrad - \emgrad\|_*
\end{align}
We use \cref{lm:pathwise martingale bound} again in a similar way of deriving \eqref{eq:expected regret for x_t}, except that this time
the MDS $\exgrad - \emgrad$ is vector-valued.
Akin to showing \eqref{eq:V} and \eqref{eq:W}, the statistics $V$ can be bounded as
\begin{align}
	V & \le \sum  \paren{2 \norm{\exgrad}_*^2 + 2 \norm{\emgrad}_*^2}  \label{eq:V_2} \\
	  & \le 8\beta (\exreg + \emreg + N\exerr + N\emerr)
\end{align}
and similarly for $W$:
\begin{align}
	W & \le \sum \paren{2\norm{\exgrad}_*^2 + 2 \E_{\emloss_n|\theta_n}[\norm{\emgrad}_*^2]} \label{eq:W_2} \\
	  & \le  16 \beta (\exreg + N\exerr)
\end{align}
%\cheng{We don't need $R_\Theta$ below right?}
Therefore,
\begin{align} \label{eq:v_w_2}
	V+W \le 24 \beta (\exreg + \emreg + N \exerr + N \emerr)
\end{align}
Furthermore, by \cref{as:G}, it can be shown from \eqref{eq:V_2} and \eqref{eq:W_2} that $V \le 4G^2 N$ and $W \le 4G^2 N$.
To invoke \cref{lm:pathwise martingale bound},
let $\KK$ in \cref{lm:pathwise martingale bound} be $\HH$, and denote the corresponding $B$ in \cref{lm:pathwise martingale bound} by $B_\HH$.
Then, for $\delta < 1/e$, with probability at least $1-\delta$, the following holds
\begin{align} \label{eq:expected regret for x^star}
	 & \quad \norm{\sum \exgrad - \emgrad}_* \nonumber                         \\
	 & \le 2B_\HH\sqrt{8\beta}\sqrt{ \exreg + \emreg + N\exerr + N\emerr}  \;+
	\nonumber                                                                  \\
	%  & \quad \sqrt{96 \beta\log(1/\delta)} \sqrt{1+1/2\log(16G^2 N+1)}
	%   \sqrt{\exreg + N\exerr + \emreg/3 + N\emerr/3+ 1/48}
	 & \quad \sqrt{96 \beta\log(1/\delta)} \sqrt{1+1/2\log(16G^2 N+1)}
	\sqrt{\exreg + N\exerr + \emreg + N\emerr+ 1/(48\beta)}
\end{align}

\subsection{Upper Bound of the Regret} \label{app:regret}

Besides analyzing the MDSs, we need to bound the regret to the linear functions defined by the gradients (the last term in \eqref{eq:proof first app}).
Since this last term is linear, not CSN, the bias-dependent online regret bound in the proof of \cref{thm:rough bound in expectation} does not apply.
Nonetheless, because these linear functions are based on the gradients of CSN functions,
we discover that their regret rate actually obeys the exact same rate as the regret to the CSN loss functions.
This is notable because the regret to these linear functions upper bounds the regret to the CSN loss functions.

\begin{lemma} \label{lm:ol gradient bias dependent rate}
	Under the same assumptions and setup in \cref{lm:ol bias dependent rate},
	\begin{align} \label{eq:ol gradient result}
		\regret(\<\gn, \cdot\>) \le 8 \beta R_\AA^2 +  \sqrt{8\beta R_\AA^2 N \embnd}.
	\end{align}
\end{lemma}
\begin{proof}
	It suffices to show a self-bounding property for $\regret(\<\gn, \cdot\>)$ as \eqref{eq:concrete self-bounding}. Once this is established, the rest resembles how \eqref{eq:ol result} follows from \eqref{eq:concrete self-bounding} through algebraic manipulations.
	%\cheng{Give their definitions here. They're too far.}
	As in \cref{lm:ol bias dependent rate}, define $\lambda = \frac{1}{2\eta}$ and	$r^2 = 2R_\AA^2$.
	Due to the property of admissible online algorithms, one can obtain
	\begin{align}
		\regret(\<\gn, \cdot\>) \le \frac{1}{\eta} R_\AA^2 + \frac{\eta}{2}\sum \norm{\gn}_*^2 =  \lambda r^2 + \sum \frac{1}{4\lambda}\norm{\gn}_*^2
	\end{align}
	To proceed, as in \cref{lm:ol bias dependent rate}, let $\emerr = \frac{1}{N}\min_{\theta \in \Theta} \sum \emlossx$ be the bias, and let $\embnd$ be such that $\embnd\geq \emerr$ almost surely.
	Using \cref{lm:self bounding} and the admissibility of online algorithm $\AA$ yields a self-bounding property for $\regret(\<\gn, \cdot\>)$:
	\begin{align*}
		\regret(\<\gn, \cdot\>) & \le \lambda r^2 +  \frac{\beta}{\lambda} \sum f_n(\theta_n)                                      \\
		                        & \le \lambda r^2 + \frac{\beta}{\lambda} \regret(f_n) + \frac{\beta}{\lambda} N \embnd            \\
		                        & \le \lambda r^2 + \frac{\beta}{\lambda} \regret(\<\gn, \cdot\>) + \frac{\beta}{\lambda} N \embnd
	\end{align*}
	This self-bounding property is exactly like what we have seen in the self-bounding property for $\regret(f_n)$.
	%	\cheng{Recall what $\hat{E}$ is }
	After rearranging and computing the optimal $\lambda$ (which coincides with the optimal $\lambda$ in \cref{lm:ol bias dependent rate}), \eqref{eq:ol gradient result} follows.
\end{proof}
\cref{lm:ol gradient bias dependent rate} provides a bias-dependent regret to the linear functions defined by the gradients when the (stepsize) constant $\eta$ is set optimally in the online algorithm $\AA$ (used in~\cref{alg:online imitation learning}).
Interestingly, the optimal $\eta$ that achieves the bias-dependent regret coincides with the one for achieving a bias-dependent regret to CSN functions. Therefore, a bias-dependent bound for $\emreg$ and $\gradreg$ can be achieved simultaneously.

%\cheng{Mention how this will be used later on. For example, this lemma needs to know an upper bound of $\hat{\epsilon}$ to set the regularization properly.}

\subsection{Putting Things Together}
We now have all the pieces to prove \cref{thm:rough bound in high probability}.
Plugging \eqref{eq:expected regret for x_t}, \eqref{eq:x_star cauchy}, and \eqref{eq:expected regret for x^star} into the decomposition \eqref{eq:proof first app}, we have, for $\delta < 1/e$, with probability at least $1-2\delta$
\begin{align*}
	 & \quad \exreg                                                                                                                                      \\
	 & \le 2B_\R \sqrt{8\beta R_\Theta^2}\sqrt{ \exreg + \emreg + N\exerr + N\emerr}
	\\
	 & +\sqrt{96\beta R_\Theta^2 \log(1/\delta)} \sqrt{1+1/2\log(16G^2R_\Theta^2 N+1)} \sqrt{\exreg+ N\exerr +\emreg + N\emerr+ 1/(48 \beta R_\Theta^2)}
	\\
	 & + 2B_\HH\sqrt{8\beta R_\Theta^2}\sqrt{ \exreg + \emreg + N\exerr + N\emerr}                                                                       \\
	 & +  \sqrt{96 \beta\log(1/\delta)} \sqrt{1+1/2\log(16G^2 N+1)}
	\sqrt{\exreg + N\exerr + \emreg + N\emerr+ 1/(48 \beta)}                                                                                             \\
	 & + \regret(\<\nabla \emloss_n(\theta_n), \cdot\>)
\end{align*}
To simplify it, we denote
%\cheng{Also they're not $O(1)$. They depend on, e.g., $R_\Theta$ and $G$, $\log(1/\delta)$ you're going to use.}
\begin{align*}
	A_1      & = 8 \max(B_\R, B_\HH) \sqrt{2 \beta R_\Theta^2},                                             \\
	A_2      & = 8\sqrt{6 \beta R_\Theta^2 \log(1/\delta)} \sqrt{1 +1/2 \log(16G^2\max(1, R_\Theta^2)N+1)}, \\
	\tilde R & = \min(1, R_\Theta)
\end{align*}
Plugging them into the above upper bound on $\exreg$ and using the basic inequality $\sqrt{a+b} \le \sqrt{a}+\sqrt{b}$ yield
\begin{align*}
	\exreg
	 & \le (A_1 + A_2) \sqrt{\exreg + \emreg} + (A_1+A_2) \sqrt{ N\exerr + N\emerr}
	\\
	 &
	+ \frac{A_2}{\sqrt{48 \beta \tilde R^2}} + \regret(\<\nabla \emloss_n(\theta_n), \cdot\>)
\end{align*}
To further simplify, using the basic inequality $\sqrt{ab} \le (a+b)/2$ yields
\begin{align*}
	\exreg & \le
	\frac{\exreg}{2} + \frac{\emreg}{2} +(A_1 + A_2) \sqrt{ N\exerr +  N\emerr}
	\\
	       & + \frac{(A_1 + A_2)^2}{2} + \frac{A_2}{\sqrt{48 \beta \tilde R^2}} + \regret(\<\nabla \emloss_n(\theta_n), \cdot\>)
\end{align*}
Rearranging terms and invoking the bias-dependent rate in  \cref{lm:ol bias dependent rate} and \cref{lm:ol gradient bias dependent rate} give
%\cheng{I think you don't need $2$ in front of $(A_1+A_2)^2$}
\begin{align} \label{eq:plug in bias-dependent rate}
	\exreg & \le \emreg +2(A_1 + A_2)\sqrt{ N\exerr +  N\emerr}+
	(A_1+A_2)^2 + \frac{A_2}{\sqrt{12 \beta \tilde R^2}} + 2\regret(\<\nabla \emloss_n(\theta_n), \cdot\>) \nonumber
	\\
	       & \le  2(A_1 + A_2)\sqrt{ N\exerr +  N\emerr} +   6\sqrt{2\beta R_\AA^2 N \embnd} +
	(A_1+A_2)^2 + \frac{A_2}{\sqrt{12 \beta \tilde R^2}} +
	24 \beta R_\AA^2
\end{align}
Finally, to derive a big-$O$ bound, denote
\begin{align*}
	R = \max(1, R_\Theta, R_\AA), \quad C = \log(1/\delta) \log(G R N)
\end{align*}
then one can obtain the rate \emph{in terms of $N$} in big-$O$ notation,  while keeping $\tilde R$, $R$, $B_\R$, $B_\HH$, $\log(1/\delta)$, $G$, $\exerr$,  and $\embnd$ as multipliers:
\begin{align*}
	\exreg = O\paren{C \beta   R^2 +
		\sqrt{C \beta   R^2 N(\embnd+ \exerr)}}
\end{align*}
Therefore
\begin{align*}
	\frac{1}{N}\sum \exlossn - \exerr  = O\paren{\frac{C  \beta R^2}{N} +
		\sqrt{\frac{C \beta  R^2  (\embnd+ \exerr)}{N}}}
\end{align*}

\blue{
	\section{ONLINE IL WITH ADAPTIVE STEPSIZES} \label{app:adaptive stepsizes}
}
In~\cref{sec:proof_thm_1} and~\cref{sec:proof_thm_2}, we proved new bias-dependent rates in expectation (\cref{thm:rough bound in expectation}) and in high-probability (\cref{thm:rough bound in high probability}).
However, these rates hold only provided that the stepsize of the online algorithm $\AA$ in~\cref{alg:online imitation learning} is constant and properly tuned; this requires knowing in advance the smoothness factor $\beta$, an upper bound of the bias $\hat \epsilon$, and the number of rounds $N$.
Therefore, these theorems are not directly applicable to practical online IL algorithms that update the stepsize adaptively without knowing the constants beforehand.

Fortunately, \cref{thm:rough bound in expectation} and~\cref{thm:rough bound in high probability} can be adapted to online IL algorithms that utilize online algorithms with adaptive stepsizes in a straightforward manner, which we shall show next.
The key insight is that online algorithms with adaptive stepsizes obtain almost the same guarantee as they would have known the optimal constant stepsize in advance.
For example, \citet[Theorem 4.14]{orabona2019modern} shows that for Online Subgradient Descent~\citep{zinkevich2003online}, the difference between the guarantees of using the optimal constant stepsize and the guarantee of using adaptive stepsizes $\eta_n = \frac{\sqrt{2}D}{2 \sqrt {\sum_{i=1}^n \|g_i\|_2^2}}$ is only a factor of $\sqrt{2}$.
\begin{lemma}[Theorem 4.14~\citep{orabona2019modern}] \label{lm:adaptive online algorithm}
	Let $V \subseteq \R^d$ a closed non-empty convex set with diameter $D$, i.e., $\max_{x, y \in V} \|x -y\|_2 \le D$. Let $f_1, \dots, f_N$ be an arbitrary sequence of non-negative convex functions $f_n: \R^d \to (-\infty, +\infty]$ differentiable in open sets containing $V$ for $t=1, \dots, T$. Pick any $x_1 \in V$ and stepsize $\eta_n = \frac{\sqrt{2}D}{2 \sqrt {\sum_{i=1}^n \|\nabla f_i(x_i\|_2^2}}$, $n = 1, \dots, N$. Then the following regret bound holds for online subgradient descent:
	\begin{align} \label{eq:adaptive online algorithm}
		\mathrm{Regret}(f_n) \le  \sqrt{2} \min_{\eta>0}\paren{\frac{D^2}{2\eta} + \frac{\eta}{2} \sum_{n=1}^T \|\nabla f_n(x_n)\|_2^2}
	\end{align}
\end{lemma}
Inspired by this insight, we propose a more general definition of admissible online algorithms (cf.~\cref{def:admissible}) and a notion of proper stepsizes:
\begin{definition}[General admissible online algorithm]
	\label{def:admissible new}
	We say an online algorithm $\AA$ is \emph{admissible} on a parameter space $\Theta$, if there exists $R_\AA\in[0,\infty)$ such that given any sequence of differentiable convex functions $f_n$ and stepsizes $\eta_n$, $\AA$ can achieve
	%\begin{align} \label{eq:regret as}
	$ \regret(f_n) \le \regret(\<\gn, \cdot\>) \le \frac{1}{\eta} R_\AA^2 + \frac{1}{2}\sum \eta_n \norm{\gn}_*^2$,
	%\end{align}
	where %$g_n = \nabla f_n(\theta_n)$ and
	$\theta_n$ is the decision made by $\AA$ in round $n$.
\end{definition}
The admissible online algorithms that we discussed in the last paragraph of~\cref{app:online regret} also belong to this category of genearal admissible algorithms.
\begin{definition}[Proper stepsizes] \label{def:proper stepsizes}
	A stepsize adaptation rule is \emph{proper} if
	there exists $K \in (0, \infty)$ such that for
	any admissible online algorithm $\AA$ (\cref{def:admissible new}) with the stepsize $\eta_n$ chosen according to the rule based on the information till round $n$
	can achieve $ \regret(f_n) \le \regret(\<\gn, \cdot\>) \le  K \min_{\eta > 0}\paren{\frac{1}{\eta} R_\AA^2 + \frac{1}{2}\sum \eta \norm{\gn}_*^2}$.
\end{definition}
With the general definition of admissible online algorithms (\cref{def:admissible new}) and the definition of proper stepsizes (\cref{def:proper stepsizes}), we can now extend the bias-dependent regret (\cref{lm:ol bias dependent rate}) which assumes optimal constant stepsize to adaptive online algorithms with proper stepsizes.
This lemma will be the foundation of extending~\cref{thm:rough bound in expectation} and~\cref{thm:rough bound in high probability}.
\begin{lemma} \label{lm:adaptive ol bias dependent rate}
	Consider running an admissible online algorithm $\AA$ on a sequence of CSN loss functions $\{f_n\}$ with adaptive stepsizes that are proper.
	Let $\{\theta_n\}$ denote the online decisions made in each round, and let $\emerr = \frac{1}{N}\min_{\theta \in \Theta} \sum f_n(\theta)$ be the bias, and let $\embnd$ be such that $\embnd\geq \emerr$ almost surely.
	Then the following holds
	%\cheng{Missing $\frac{1}{N}$ is the definition of $\hat{E}$?}
	\begin{align*}
		\regret(f_n) \le 8 K^2 \beta R_\AA^2 +  \sqrt{8 K^2\beta R_\AA^2 N \embnd}.
	\end{align*}
\end{lemma}
\begin{proof}
	%\cheng{define what $\gn$ is again or maybe it's better to just write it out here}
	Because the online algorithm $\AA$ is admissible and the stepsizes are proper, we have, for \emph{any} $\eta > 0$
	\begin{align} \label{eq:adaptive constant stepsize}
		\regret(f_n)  \le \frac{K}{\eta} R_\AA^2 + \frac{K \eta}{2}\sum \norm{\gn}_*^2
	\end{align}
	Let $\lambda = \frac{1}{2\eta}$ and $r^2 = 2R_\AA^2$, then
	\begin{align}
		\frac{K}{\eta} R_\AA^2 + \frac{K \eta}{2}\sum \norm{\gn}_*^2
		=  K \lambda r^2 + \sum \frac{K}{4\lambda}\norm{\gn}_*^2
	\end{align}
	Using \cref{lm:self bounding} yields a \emph{self-bounding property} for $\regret(f_n)$:
	%\cheng{Missing $\frac{\beta}{\lambda}$ on $\hat{E}$?}
	\begin{align} \label{eq:adaptive concrete self-bounding}
		\regret(f_n) \le K \lambda r^2 +  \frac{K\beta}{\lambda} \sum f_n(\theta_n) \le K\lambda r^2 + \frac{K\beta}{\lambda} \regret(f_n) + \frac{K\beta}{\lambda} N \embnd
	\end{align}
	Let $\hat \beta = K\beta$ and $\hat r^2 =K r^2$, and
	by rearranging the terms, we have a bias-dependent upper bound (cf.~\eqref{eq:proof thm1 bound}) that for any $\eta> 0$,
	\begin{align} \label{eq:adaptive proof thm1 bound}
		\regret(f_n) \le \frac{\hat \beta}{\lambda -  \hat \beta} N\embnd + \frac{\lambda^2}{\lambda - \hat \beta} \hat r^2
	\end{align}
	The upper bound can be minimized by choosing an optimal $\lambda$.
	Setting the derivative of the right-hand side to zero, and computing the optimal $\lambda$ ($\lambda > 0$) gives us
	\begin{align} \label{eq:adaptive lambda}
		\hat r^2 \lambda^2 - 2 \hat \beta \hat r^2 \lambda - \hat \beta N \embnd = 0, \quad \lambda > 0 \quad \text{ and } \quad \lambda = \hat \beta + \sqrt{\hat \beta^2 + \frac{\hat \beta N \embnd}{\hat r^2}}
	\end{align}
	which implies that the optimal $\eta$ is $\frac{1}{2\paren{\hat \beta + \sqrt{\hat \beta^2 + \frac{\hat \beta N \embnd}{2 R_\AA^2}}}}$.
	Because~\eqref{eq:adaptive proof thm1 bound} holds for any $\eta$, it holds for the optimal $\eta$ too.
	Next, we simplify~\eqref{eq:adaptive proof thm1 bound}.
	Since the optimal $\lambda$ satisfies the equality $\hat \beta N \embnd = \hat r^2 \lambda^2 - 2 \hat \beta \hat r^2 \lambda$ implied from \eqref{eq:adaptive lambda}, \eqref{eq:adaptive proof thm1 bound} can be written as
	\begin{align}
		\regret(f_n) & \le \frac{1}{\lambda-\hat \beta} \hat \beta N \embnd + \frac{\lambda^2}{\lambda - \hat \beta} \hat r^2 =
		\frac{1}{\lambda - \hat \beta}(\hat r^2 \lambda^2 - 2 \hat \beta \hat r^2 \lambda) +\frac{\lambda^2}{\lambda - \hat \beta} \hat r^2 \nonumber \\
		             & =\frac{2\lambda^2 \hat r^2 - 2\hat \beta \lambda \hat r^2}{\lambda - \hat \beta}
		= 2 \lambda \hat r^2
	\end{align}
	Plugging in the optimal $\lambda$ yields
	%\cheng{I don't get the first inequality; why $\hat{\epsilon}$ suddenly appears, what is it? write out what the last inequality is. $\sqrt{a+b}\leq\sqrt{a}+\sqrt{b}$? }
	\begin{align} \label{eq:adaptive ol result}
		\regret(f_n) & \le 2 \lambda \hat r^2
		= 2 \paren{\hat \beta + \sqrt{\hat \beta^2 + \frac{\hat \beta N\embnd }{\hat r^2}}} \hat r^2\nonumber                  \\
		             & = 2 \hat \beta \hat r^2 + \sqrt{2\hat \beta \hat r^2}  \sqrt{2\hat \beta \hat r^2 + 2N\embnd} \nonumber \\
		             & \le 4 \hat \beta \hat r^2 + 2 \sqrt{\hat \beta \hat r^2 N \embnd} \nonumber                             \\
		             & = 8 K^2 \beta R_\AA^2 +  \sqrt{8K^2  \beta R_\AA^2 N \embnd}
	\end{align}
	where the last inequality uses the basic inequality: $\sqrt{a+b}\leq\sqrt{a}+\sqrt{b}$.
\end{proof}
Provided the bias-dependent regret~\cref{lm:adaptive ol bias dependent rate} (cf.~\cref{lm:ol bias dependent rate}),
a bias-dependent rate in expectation for online IL with an additional constant $K$ due to the adaptive stepsizes (cf.~\cref{thm:rough bound in expectation}) follows directly, because~\cref{lm:vanish} still holds even if the stepsizes of the online algorithm in~\cref{alg:online imitation learning} become adaptive.
In order to extend~\cref{thm:rough bound in high probability} to admissible online algorithm with adaptive stepsizes,
we first need to derive a bias-dependent regret to the linear functions defined by the gradients (cf.~\cref{lm:ol gradient bias dependent rate}) based on the proof of~\cref{lm:adaptive ol bias dependent rate}.
\begin{lemma} \label{lm:adaptive ol gradient bias dependent rate}
	Under the same assumptions and setup in \cref{lm:adaptive ol bias dependent rate},
	\begin{align} \label{eq:adaptive ol gradient result}
		\regret(\<\gn, \cdot\>) \le 8 K^2 \beta R_\AA^2 +  \sqrt{8K^2 \beta R_\AA^2 N \embnd}.
	\end{align}
\end{lemma}
\begin{proof}
	It suffices to show a self-bounding property for $\regret(\<\gn, \cdot\>)$ as \eqref{eq:adaptive concrete self-bounding}. Once this is established, the rest resembles how \eqref{eq:adaptive ol result} follows from \eqref{eq:adaptive concrete self-bounding} through algebraic manipulations.
	%\cheng{Give their definitions here. They're too far.}
	As in \cref{lm:adaptive ol bias dependent rate}, define $\lambda = \frac{1}{2\eta}$ and	$r^2 = 2R_\AA^2$.
	Due to the property of admissible online algorithms, one can obtain, for any $\eta$
	\begin{align}
		\regret(\<\gn, \cdot\>) \le \frac{K}{\eta} R_\AA^2 + \frac{K\eta}{2}\sum \norm{\gn}_*^2 =  K \lambda r^2 + \sum \frac{K}{4\lambda}\norm{\gn}_*^2
	\end{align}
	To proceed, as in \cref{lm:ol bias dependent rate}, let $\emerr = \frac{1}{N}\min_{\theta \in \Theta} \sum \emlossx$ be the bias, and let $\embnd$ be such that $\embnd\geq \emerr$ almost surely.
	Using \cref{lm:self bounding} and the admissibility of online algorithm $\AA$ yields a self-bounding property for $\regret(\<\gn, \cdot\>)$:
	\begin{align*}
		\regret(\<\gn, \cdot\>) & \le K \lambda r^2 +  \frac{K \beta}{\lambda} \sum f_n(\theta_n)                                     \\
		                        & \le K \lambda r^2 + \frac{K\beta}{\lambda} \regret(f_n) + \frac{K\beta}{\lambda} N \embnd           \\
		                        & \le K\lambda r^2 + \frac{K\beta}{\lambda} \regret(\<\gn, \cdot\>) + \frac{K\beta}{\lambda} N \embnd
	\end{align*}
	This self-bounding property is exactly like what we have seen in the self-bounding property for $\regret(f_n)$.
	%	\cheng{Recall what $\hat{E}$ is }
	After rearranging and computing the optimal $\lambda$ (which coincides with the optimal $\lambda$ in \cref{lm:ol bias dependent rate}), \eqref{eq:adaptive ol gradient result} follows.
\end{proof}
Given the bias-dependent rates in online learning (\cref{lm:adaptive ol bias dependent rate} and~\cref{lm:adaptive ol gradient bias dependent rate}),
a high-probability bias-dependent rate for online IL with an additional constant $K$ due to adaptive stepsizes (cf.~\cref{thm:rough bound in high probability}) can the derived in the same way as the proof of~\cref{thm:rough bound in high probability}, except that~\cref{lm:adaptive ol bias dependent rate} and~\cref{lm:adaptive ol gradient bias dependent rate} will be invoked in~\eqref{eq:plug in bias-dependent rate} in place of~\cref{lm:ol bias dependent rate} and~\cref{lm:ol gradient bias dependent rate}.

Interestingly, \citet[Theorem 4.21]{orabona2019modern} provides a bias-dependent regret for Online Subgradient Descent with adaptive stepsizes $\eta_n = \frac{\sqrt{2}D}{2 \sqrt {\sum_{i=1}^n \|g_i\|_2^2}}$ (cf.~\cref{lm:adaptive ol bias dependent rate}).
Although that regret bound would help more directly prove~\cref{thm:rough bound in expectation} in the adaptive stepsize setting, but it does not directly imply a bias-dependent regret to the linear functions defined by the gradients (cf.~\cref{lm:adaptive ol gradient bias dependent rate})

\section{EXPERIMENT DETAILS} \label{app:exp}
Although the main focus of this paper is the new theoretical insights, we conduct experiments to provide evidence that the fast policy improvement phenomena indeed exist, as our theory predicts. We verify the change of the policy improvement rate due to policy class capacity by running an online IL experiment in the CartPole balancing task in OpenAI Gym~\citep{brockman2016openai} with DART physics engine~\citep{Lee2018}.

\subsection{MDP Setup}
The goal of the CartPole balancing task is to keep the pole upright by controlling the acceleration of the cart.
This MDP has a 4-dimensional continuous state space (the position and the velocity of the cart and the pole), and 1-dimensional continuous action space (the acceleration of the cart).
The initial state is a configuration with a small uniformly sampled offset from being static and vertical, and the dynamics is deterministic.
This task has a maximum horizon of $1000$.
In each time step, if the pole is maintained within a threshold from being upright, the learner receives an instantaneous reward of one; otherwise, the learner receives zero reward and the episode terminates.
Therefore, the maximum sum of rewards for an episode is $1000$.

\subsection{Expert Policy Representation and Training}
To simulate the online IL task, we consider a neural network expert policy (with one hidden layer of $64$ units and $\mathrm{tanh}$ activation), and the inputs to the neural network is normalized using a moving average over the samples.
The expert policy is trained using a model-free policy gradient method (ADAM~\citep{kingma2014adam} with GAE \citep{schulman2015high}).
And the value function used by GAE is represented by a neural network with two hidden layers of $128$ units and $\mathrm{tanh}$ activation.
To compute the policy gradient during training, additional Gaussian noise (with zero mean and a learnable variance that does not depend on the state) is added to the actions, and the gradient is computed through log likelihood ratio.
After $100$ rounds of training, the expert policy can consistently achieve the maximum sum of rewards both with and without the additional Gaussian noise.
After the expert policy is trained, during online IL, Gaussian noise is not added in order to reduce the variance in the experiments.

\subsection{Learner Policy Representation} \label{sec:learner}
We let the learner policy be another neural network that has exactly the same architecture as the expert policy with no Gaussian noise added.
\blue{
	In the setting of only training the output layer, we copy the weights for the hidden layer and the input normalizer from those of the expert policy and randomly initialize the weights of the output layer.
	During training, only the weights of the learner's output layer were updated.
	In this way, we can view the learner as a \emph{linear} policy using the representation of the expert policy.
	In the setting of training the full network, we still copy the input normalizer from that of the expert policy but we randomly initialize all the variables in the network, i.e., weights and biases of the hidden and output layers.
	During training, all of these variables were updated.
}
\subsection{Online IL Setup}
\paragraph{Policy class}
We conduct online IL with unbiased and biased policy classes.
One one hand, we define the unbiased class as all the policies satisfying the representation in \cref{sec:learner}.
On the other hand, we define the biased policy classes by imposing an additional $\ell_2$-norm constraint \blue{with different sizes} on the learner's weights in the output layer so that the learner cannot perfectly mimic the expert policy.
More concretely, in the experiments, the $\ell_2$-norm constraint has sizes $\{0.1, 0.12, 0.15\}$. \blue{This set of constraints was chosen based on the observation that the $\ell_2$-norm of the final policy trained without the constraint is about $0.18$ when training the output layer only and about $0.23$ when training the full network.}

\paragraph{Loss functions}
We select $\exlossx = \E_{s \sim d_{\pi_{\theta_n}}} [H_\mu(\pi_\theta(s) - \expert(s))]$ as the online IL loss (see \cref{sec:reducing policy optimization}), where $H_\mu$ is the Huber function %(implemented in Keras)
defined as
$H_\mu(x) = \frac{1}{2} x^2$ for $\abs{x} \le \mu$ and $
	\mu \abs{x} - \frac{1}{2}\mu^2$ for $\abs{x} > \mu$.
In the experiments,
$\mu$ is set to $0.05$; as a result, $H_\mu$ is linear when its function value is larger than $0.00125$.
Because the learner's policy is linear, this online loss is CSN in the unknown weights of the learner.

\paragraph{Policy update rule}
We choose AdaGrad~\citep{mcmahan2010adaptive,duchi2011adaptive} as the online algorithm in~\cref{alg:online imitation learning}; AdaGrad is a first-order mirror descent algorithm \blue{and well matches the assumptions made in our theorems (\cref{app:adaptive stepsizes})},
When the $\ell_2$-norm constraint is imposed, an additional projection step is taken after taking a gradient step using AdaGrad.
The final algorithm is a special case of the DAgger algorithm~\citep{ross2011reduction} (called DAggereD in~\citep{cheng2018fast}) with only first-order information and continuous actions~\citep{cheng2018fast}.
In the experiments, the stepsize is set to $0.01$.
In each round, for updating the learner policy, $1000$ samples, i.e., state and expert action pairs, are gathered, and for computing the loss $\exloss_n(\theta_n)$, more samples ($5000$ samples) are used due to the randomness in the initial state of the MDP.
\blue{The total number of iterations is 500 for both the training output layer and the training full network experiments. Due to the randomness in the initial state of the MDP and the initialization of the policy, we averaged the results over 4 random seeds.
}

\paragraph{Hyperparameter tuning}
The hyperparameters are tuned in a very coarse manner. We eliminated the ones that are obviously not proper.
Here are the hyperparameters we have tried.
The stepsize in online IL: $0.1, 0.01, 0.001$.
The Huber function parameter $\mu$: $0.05$.

\end{document}